%% file: main.tex
\documentclass[sigconf,9pt,anonymous=false]{acmart}

\pdfoutput=1 

\font\stixfrak=stix-mathfrak at 9pt
-mathcal at 9pt

\usepackage{graphicx}

\usepackage{subfigure}
\usepackage{amsmath}
\usepackage{amsfonts}
\usepackage{mathrsfs}
\usepackage{amsxtra}
\usepackage{amsthm}

\usepackage{dutchcal}

\usepackage{enumitem}

\usepackage{float}

\usepackage{calc}

\usepackage{cancel}

\usepackage{amsbsy}
\usepackage{cleveref}
\usepackage{mathrsfs}
\usepackage[linesnumbered,lined,boxed,commentsnumbered]{algorithm2e}

\SetCommentSty{mycommfont}
\usepackage{multicol}

\makeatletter
\newcommand\Func[2]{%
    \textbf{function} #1
    \algocf@group{#2}%
}
\makeatother

\makeatletter
\newcommand\Forr[2]{%
    \textbf{for} #1 \textbf{do}%
    \algocf@group{#2}%
}
\makeatother

\makeatletter
\newcommand\Blnk[2]{%
    \hspace{10pt}#1
    \algocf@group{#2}%
    \textbf{end}
}
\makeatother

\crefname{algocf}{alg.}{algs.}
\Crefname{algocf}{Algorithm}{Algorithms}

\newtheorem{theorem}{Theorem}
\newtheorem{definition}{Definition}

\newtheorem{remark}{Remark}

\newtheorem{proposition}{Proposition}

\newtheorem{problem}{Problem}
\newtheorem{subproblem}{Problem}[problem]

\usepackage{soul}

\usepackage{color}

\newcommand{\nn}{{\mathscr{N}\negthickspace\negthickspace\negthinspace\mathscr{N}}\negthinspace}
\newcommand{\overbar}[1]{\mkern 3.0mu\overline{\mkern-2.5mu#1\mkern-2.0mu}\mkern 2.0mu}

\newcommand{\lblkbrbrack}{\negthinspace\text{{\stixfrak\char"36}}\normalfont}
\newcommand{\rblkbrbrack}{\text{{\stixfrak\char"37}}\normalfont}

\newcommand{\subarg}[1]{\lblkbrbrack #1 \rblkbrbrack}

\newcommand{\myalg}{Fast BATLLNN}

\acmConference[ 2021]{ 2021}{USA}

\title{Fast BATLLNN: Fast Box Analysis of Two-Level Lattice Neural Networks}
\author{James Ferlez}
\affiliation{
	\institution{University of California, Irvine}
	\department{Dept. of Electrical Engineering and Computer Science}
	\country{USA}
}
\email{jferlez@uci.edu}
\author{Haitham Khedr}
\affiliation{
	\institution{University of California, Irvine}
	\department{Dept. of Electrical Engineering and Computer Science}
	\country{USA}
}
\email{hkhedr@uci.edu}
\author{Yasser Shoukry}
\affiliation{
	\institution{University of California, Irvine}
	\department{Dept. of Electrical Engineering and Computer Science}
	\country{USA}
}
\email{yshoukry@uci.edu}



\begin{document}

\begin{abstract}
	In this paper, we present the tool Fast Box Analysis of Two-Level Lattice 
	Neural Networks (\myalg) as a fast verifier of box-like output constraints 
	for Two-Level Lattice (TLL) Neural Networks  (NNs). In particular, \myalg~ 
	can verify whether the output of a given TLL NN always lies within a 
	specified hyper-rectangle whenever its input constrained to a specified 
	convex polytope (not necessarily a hyper-rectangle). \myalg~uses the unique 
	semantics of the TLL architecture and the decoupled nature of box-like 
	output constraints to dramatically improve verification performance 
	relative to known polynomial-time verification algorithms for TLLs with 
	generic polytopic output constraints.
	In this paper, we evaluate the performance and scalability of \myalg, both 
	in its own right and compared to state-of-the-art NN verifiers applied to  
	TLL NNs. \myalg~compares very favorably to even the fastest NN verifiers, 
	completing our synthetic TLL test bench more than 400x faster than its 
	nearest competitor. 

\end{abstract}

\maketitle

\input{intro.tex}
\input{prelims.tex}

\input{problem}

\input{algorithm}

\input{implementation}

\input{experiments}
\begin{acks}
	This work was supported by the \grantsponsor{}{National Science 
	Foundation}{https://www.nsf.gov} under grant numbers \grantnum{}{\#2002405} 
	and \grantnum{}{\#2013824}. 
\end{acks}

\newpage

\bibliographystyle{ACM-Reference-Format} %
\bibliography{mybib}







\end{document}

%% file: intro.tex

\section{Introduction} 
\label{sec:introduction}
Neural Networks (NNs) increasingly play vital roles within safety-critical 
cyber-physical systems (CPSs), where they either make safety-critical decisions 
directly (as in the case of low-level controllers) or influence high-level 
supervisory decision making (e.g. through vision networks).
Ensuring the safety of such systems thus demands algorithms capable of  
\emph{formally verifying} the safety of their NN components. However, since CPS 
safety is characterized by \emph{closed-loop} behavior, it is not enough to 
pragmatically verify the input/output behavior of a NN component \emph{once}. 
Such a verifier must additionally be as fast as possible, so that it can 
feasibly be invoked many times during the course of verifying a closed-loop 
property (as in \cite{TranNNVNeuralNetwork2020, WangNeuralNetworkControl2020} 
for example).

In this paper, we propose \myalg~as an input/output verifier for Rectified 
Linear Unit (ReLU) NNs with a special emphasis on execution time. In 
particular, \myalg~takes a relatively uncommon approach among verifiers in that 
it \textbf{explicitly trades off generality for execution time}: whereas most 
NN verifiers are designed to work for arbitrary deep NNs and arbitrary 
half-space output properties (or the intersections thereof) \cite{vnn2020}, 
\myalg~instead forgoes this generality in network and properties to reduce 
verification time. That is, \myalg~is only able to verify \emph{a very specific 
subset of deep NNs}: those characterized by a particular architecture, the 
Two-Level Lattice (TLL) NN architecture introduced in 
\cite{FerlezAReNAssuredReLU2020}; see also \Cref{fig:tll_arch} and 
\Cref{sub:two_layer_lattice_neural_networks}. Similarly, \myalg~is restricted 
to verifying \emph{only ``box''-like output constraints} (formally, 
hyper-rectangles).
\begin{figure}[b]
	\centering %
	\includegraphics[width=0.42\textwidth]{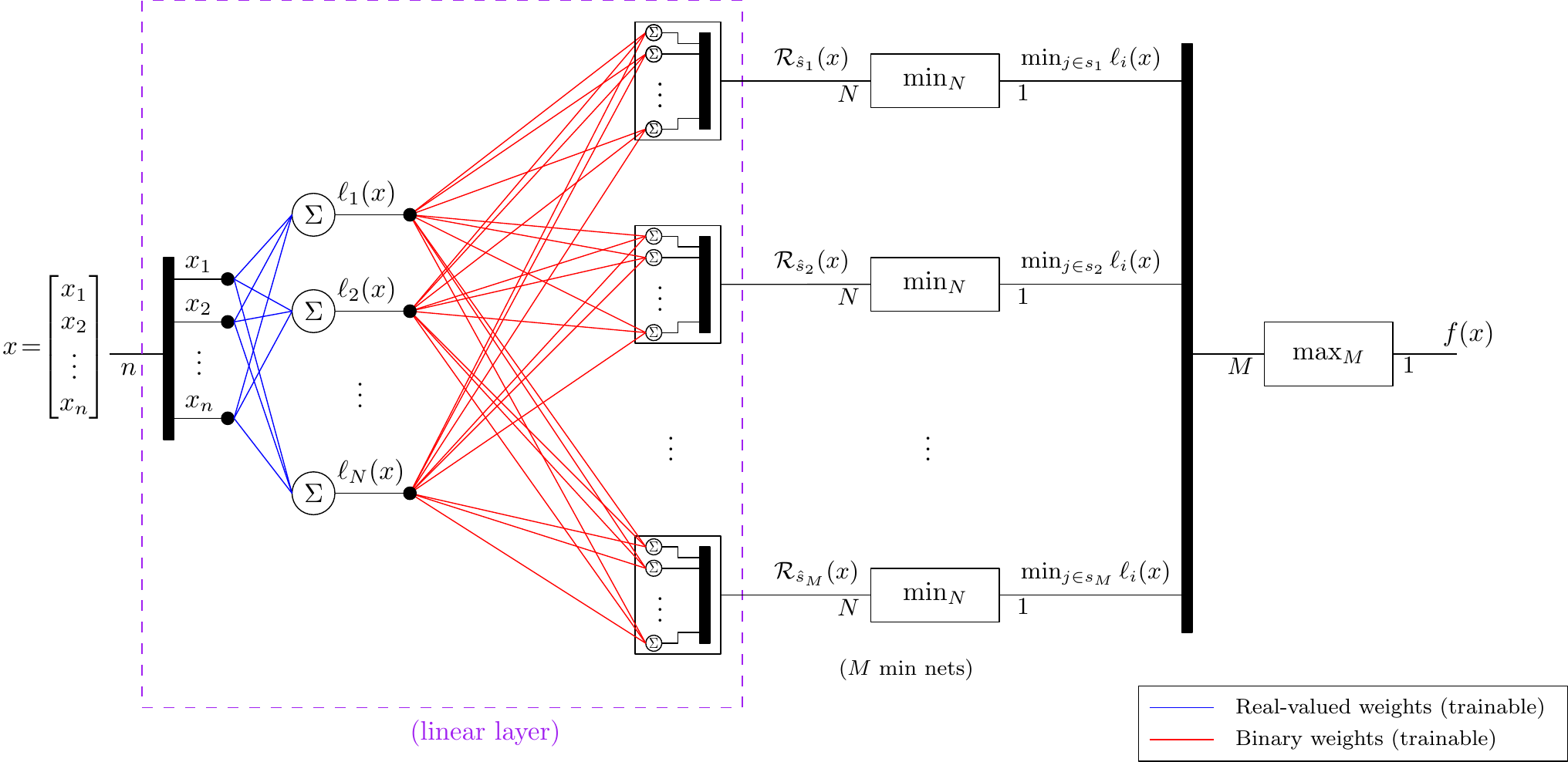} %
	\caption{Illustration of TLL NN architecture for a function $\mathbb{R}^n \rightarrow \mathbb{R}$; reproduced from \cite{FerlezAReNAssuredReLU2020}. More details in \Cref{sub:two_layer_lattice_neural_networks}.} %
	\label{fig:tll_arch} %
\end{figure}
Through extensive experiments, \myalg~exemplifies that sacrificing generality 
in both of these senses can lead to dramatically faster verification times. 
\emph{Compared to state-of-the-art general NN verifiers, \myalg~is 400-1900x 
faster verifying the same TLL NNs and properties.}

In this sense, \myalg~ is primarily inspired by the recent result 
\cite{FerlezBoundingComplexityFormally2020}, which showed that verifying a 
Two-Level Lattice (TLL) NN is an ``easier'' problem than verifying a general 
deep NN. Specifically, \cite{FerlezBoundingComplexityFormally2020} exhibits a 
\emph{polynomial time} algorithm to verify a TLL  with respect to an arbitrary 
half-space output property (i.e. polynomial-time in the number of neurons).  
Indeed, the \emph{semantic} structure of the TLL architecture is precisely what 
makes polynomial-time verification possible: in a TLL NN, the neuronal 
parameters provide direct (polynomial-time) access to each of the affine 
functions that appear in its response, viewed as a Continuous Piecewise Affine 
(CPWA) function\footnote{Recall that any ReLU NN implements  a CPWA: i.e., a 
function that continuously switches between finitely many affine functions.} 
\cite{FerlezBoundingComplexityFormally2020}. Since the same cannot be said of 
the neuronal parameters in a general deep NN, this indicates that  considering 
only TLL NNs can facilitate a much faster verifier.

Thus, the major contribution of \myalg~is to further leverage the semantics of 
the TLL architecture under the additional assumption of verifying box-type (or 
hyper-rectangle) output properties. In particular, a TLL NN implements  
(component-wise) $\min$ and $\max$ lattice operations to compute each of its 
real-valued output components (as illustrated in \Cref{fig:tll_arch}; see also 
\Cref{sub:two_layer_lattice_neural_networks}). This fact can be used to 
dramatically simplify the verification of box-like output properties, which are 
component-wise \emph{real-valued intervals} -- and hence mutually decoupled. 
Importantly, the algorithm proposed in 
\cite{FerlezBoundingComplexityFormally2020} cannot take advantage of these  
lattice operations in the same way, since it considers only general half-space 
properties, which naturally \emph{couple} the various output components of the 
TLL NN. As a result, we can show that \myalg~has a big-O verification 
complexity whose crucial exponent is \emph{half the size of the analogous 
exponent in} \cite{FerlezBoundingComplexityFormally2020}. The performance 
consequences of this improvement are reflected in our experimental results.

Before we proceed further, it is appropriate to make a few remarks about the 
restrictions inherent to \myalg. Between the two restrictions of significance 
-- the restriction to TLL NNs and the restriction to box-like output properties 
-- the former is apparently more onerous: box-like properties can be used to 
adaptively assess more complicated properties whenever box-like properties are 
themselves inadequate. However even the restriction to TLL NNs is less imposing 
than at first it may seem. On the one hand, it is known that TLL NNs are 
capable of representing any Continuous Piecewise-Affine (CPWA) function 
\cite{FerlezAReNAssuredReLU2020, TarelaRegionConfigurationsRealizability1999}; 
i.e., any function that continuously switches between a finite set of affine 
functions. Since deep NNs themselves realize CPWA functions, the TLL NN 
architecture is able to instantiate any function that a generic deep NN can.  
We do not consider the problem of converting a deep NN to the TLL architecture 
(nor the possible loss in parametric efficiency that may result), but the 
extremely fast verification times achievable with \myalg~suggest that the trade 
off is very likely worth the cost. On the other hand, there is a spate of 
results which suggest that the TLL architecture a useful architecture within 
which to do  closed-loop controller design in the first place  
\cite{FerlezAReNAssuredReLU2020, CruzSafebyRepairConvexOptimization2021, 
FerlezTwoLevelLatticeNeural2020} -- potentially obviating the need for such a 
conversion at all.

\noindent \textbf{Related work:} 
To the best of our knowledge, the work 
\cite{FerlezBoundingComplexityFormally2020} is the only current example of an 
attempt to verify a specific NN architecture rather than a generic deep NN.  

The literature on more general NN verifiers is far richer. These NN verifiers 
can generally be grouped into three categories: (i) SMT-based methods, which 
encode the problem into a Satisfiability Modulo Theory 
problem~\cite{ehlers2017formal,KatzReluplexEfficientSMT2017a,katz2019marabou}; 
(ii) MILP-based solvers, which directly encode the verification problem as a 
Mixed Integer Linear 
Program~\cite{anderson2020strong,bastani2016measuring, bunel2020branch, cheng2017maximum,fischetti2018deep,lomuscio2017approach,tjeng2017evaluating};
(iii) Reachability based methods, which perform layer-by-layer reachability 
analysis to compute the reachable
set~\cite{BakImprovedGeometricPath2020,fazlyab2019efficient,gehr2018ai2,ivanov2019verisig,tran2020nnv,wang2018formal,xiang2017reachable,xiang2018output};
and (iv) convex relaxations 
methods~\cite{dvijotham2018dual,wang2018efficient,wong2017provable,KhedrPEREGRiNNPenalizedRelaxationGreedy2020}.
Methods in categories (i) - (iii) tend to suffer from poor scalability,  
especially relative to convex relaxation methods. In this paper, we perform 
comparisons with  state-of-the-art examples from category (iv) 
\cite{KhedrPEREGRiNNPenalizedRelaxationGreedy2020, wang2021betacrown} and 
category (iii) \cite{BakImprovedGeometricPath2020}, as perform well overall in 
the standard verifier competition \cite{vnn2020}.

%% file: prelims.tex
\section{Preliminaries} 
\label{sec:preliminaries}

\subsection{Notation} 
\label{sub:notation}
We will denote the real numbers by $\mathbb{R}$. For an $(n \times m)$ matrix 
(or vector), $A$, we will use the notation $\llbracket A \rrbracket_{[i,j]}$ to 
denote the element in the $i^\text{th}$ row and $j^\text{th}$ column of $A$. 
Analogously, the notation $\llbracket A \rrbracket_{[i,:]}$ will denote the 
$i^\text{th}$ row of $A$, and $\llbracket A \rrbracket_{[:, j]}$ will denote 
the $j^\text{th}$ column of $A$; when $A$ is a vector instead of a matrix, both 
notations will return a scalar corresponding to the corresponding element in 
the vector.
We will use bold parenthesis $\;\subarg{ \cdot }$ to delineate the arguments to 
a function that \emph{returns a function}. We use two special forms of this 
notation: for an $(m \times n)$ matrix, $W$, and an $(m \times 1)$ vector, $b$: 
\begin{align}
\mathscr{L} \subarg{ W, b } &: x \mapsto W x + b \\
\mathscr{L}_i \subarg{ W, b } &: x \mapsto \llbracket W \rrbracket_{[i,:]} x + \llbracket b \rrbracket_{[i,1]}.
\end{align}
We also use the functions $\mathtt{First}$ and $\mathtt{Last}$ to return the 
first and last elements of an ordered list (or by overloading, a vector in 
$\mathbb{R}^n$). The function $\mathtt{Concat}$ concatenates two ordered lists, 
or by overloading, concatenates two vectors in $\mathbb{R}^n$ and 
$\mathbb{R}^m$ along their (common) nontrivial dimension to get a third vector 
in $\mathbb{R}^{n+m}$. Finally, an over-bar indicates (topological) closure of 
a set: i.e. $\overbar{A}$ is the closure of $A$. 

\subsection{Neural Networks} 
\label{sub:neural_networks}
We will exclusively consider Rectified Linear Unit Neural Networks (ReLU NNs). 
A $K$-layer ReLU NN is specified by $K$ \emph{layer} functions,  
$\{L_{\theta^{|i}} : i = 1, \dots, K\}$, of which we allow two kinds: linear 
and nonlinear. A \emph{nonlinear} layer is a function:
\begin{equation}
	L_{\theta^{|i}} : \mathbb{R}^{\mathfrak{i}} \rightarrow \mathbb{R}^{\mathfrak{o}},  \qquad
	    L_{\theta^{|i}} :  z \mapsto \max\{ W z + b, 0 \}
\end{equation}
where the $\max$ function is taken element-wise, $W^{|i}$ and $b^{|i}$ are 
matrices of appropriate dimensions, and $\theta^{|i} \triangleq 
(W^{|i},b^{|i})$. A \emph{linear} layer is the same as a nonlinear layer, 
except it omits the nonlinearity $\max\{\cdot , 0\}$ in its layer function; a 
linear layer will be indicated with a superscript \emph{lin} e.g. 
$L^\text{lin}_{\theta^{|i}}$
Thus, a $K$-layer ReLU NN function is specified by functionally composing $K$ 
such layer functions 
whose input and output dimensions satisfy $\mathfrak{i}_{i} = 
\mathfrak{o}_{i-1}: i = 2, \dots, K$. \emph{We always consider the final layer 
to be a linear layer}, so we may define:
\begin{equation}
	\nn = L_{\theta^{|K}}^\text{lin} \circ L_{\theta^{|K-1}}^{\vphantom{\text{lin}}} \circ \dots \circ
		L_{\theta^{|1}}^{\vphantom{\text{lin}}}
\end{equation}
To make the dependence on parameters explicit, we will index a ReLU function 
$\nn$ by a \emph{list of matrices} $\Theta \triangleq$ $( \theta^{|1},$ 
$\dots,$ $\theta^{|K} );$\footnote{That is $\Theta$ is not the concatenation of 
the $\theta^{|i}$ into a single large matrix, so it preserves information about 
the sizes of the constituent $\theta^{|i}$.} in this respect, we will often use 
$\nn = \nn\subarg{\Theta}$. 




\subsection{Two-Level-Lattice (TLL) Neural Networks} 
\label{sub:two_layer_lattice_neural_networks}
In this paper, we are exclusively concerned Two-Level Lattice (TLL) ReLU NNs  
as noted above. In this subsection, we formally define NNs with the TLL 
architecture using the succinct method exhibited in 
\cite{FerlezBoundingComplexityFormally2020}; the material in this subsection is 
derived from \cite{FerlezAReNAssuredReLU2020, 
FerlezBoundingComplexityFormally2020}.

The most efficient way to characterize a TLL NN is by way of three generic NN 
composition operators. Hence, we have the following three definitions, which 
serve as auxiliary results in order to eventually define a TLL NN in 
\Cref{def:scalar_tll}.
\begin{definition}[Sequential (Functional) Composition]
\label{def:functional_composition}
	Let $\nn\subarg{\Theta_{\scriptscriptstyle 1}}: \mathbb{R}^\mathfrak{i}  
	\rightarrow \mathbb{R}^\mathfrak{c}$ and 
	$\nn\subarg{\Theta_{\scriptscriptstyle 2}} : \mathbb{R}^\mathfrak{c}  
	\rightarrow \mathbb{R}^\mathfrak{o}$ be two NNs. 
	Then the \textbf{sequential (or functional) composition} of 
	$\nn\subarg{\Theta_{\scriptscriptstyle 1}}$ and 
	$\nn\subarg{\Theta_{\scriptscriptstyle 2}}$, i.e.  
	$\nn\subarg{\Theta_{\scriptscriptstyle 1}} \circ  
	\nn\subarg{\Theta_{\scriptscriptstyle 2}}$, is a well defined NN, and can 
	be represented by the parameter list $\Theta_{1} \circ \Theta_{2} 
	\triangleq \mathtt{Concat}(\Theta_1, \Theta_2)$.
\end{definition}
\begin{definition}
	\label{def:parallel_composition}
	Let $\nn\subarg{\Theta_{\scriptscriptstyle 1}}$ and  
	$\nn\subarg{\Theta_{\scriptscriptstyle 2}}$ be two $K$-layer NNs with 
	parameter lists:
	\begin{equation}
		\Theta_i = ((W^{\scriptscriptstyle |1}_i, b^{\scriptscriptstyle |1}_i), \dots, (W^{\scriptscriptstyle |K}_i, b^{\scriptscriptstyle |K}_i)), \quad i = 1,2.
	\end{equation}
	Then the \textbf{parallel composition} of  
	$\nn\subarg{\Theta_{\scriptscriptstyle 1}}$ and  
	$\nn\subarg{\Theta_{\scriptscriptstyle 2}}$ is a NN given by the parameter 
	list
	\begin{equation}
		\Theta_{1} \parallel \Theta_{2} \triangleq \big(\negthinspace
			\left(
				\negthinspace
				\left[
					\begin{smallmatrix}
						W^{\scriptscriptstyle |1}_1 \\
						W^{\scriptscriptstyle |1}_2
					\end{smallmatrix}
				\right],
				\left[
					\begin{smallmatrix}
						b^{\scriptscriptstyle |1}_1 \\
						b^{\scriptscriptstyle |1}_2
					\end{smallmatrix}
				\right]
				\negthinspace
			\right),
			{\scriptstyle \dots},
			\left(
				\negthinspace
				\left[
					\begin{smallmatrix}
						W^{\scriptscriptstyle |K}_1 \\
						W^{\scriptscriptstyle |K}_2
					\end{smallmatrix}
				\right],
				\left[
					\begin{smallmatrix}
						b^{\scriptscriptstyle |K}_1 \\
						b^{\scriptscriptstyle |K}_2
					\end{smallmatrix}
				\right]
				\negthinspace
			\right)
		\negthinspace\big).
	\end{equation}
	That is $\Theta_{1} \negthickspace \parallel \negthickspace \Theta_{2}$ 
	accepts an input of the same size as (both) $\Theta_1$ and $\Theta_2$, but 
	has as many outputs as $\Theta_1$ and $\Theta_2$ combined.
\end{definition}

\begin{definition}[$n$-element $\min$/$\max$ NNs]
	\label{def:n-element_minmax_NN}
	An $n$\textbf{-element $\min$ network} is denoted by the parameter list 
	$\Theta_{\min_n}$. $\nn\subarg{\Theta_{\min_n}}: \mathbb{R}^n \rightarrow 
	\mathbb{R}$ such that $\nn\subarg{\Theta_{\min_n}}(x)$ is the minimum 
	from among the components of $x$ (i.e. minimum according to the usual order 
	relation $<$ on $\mathbb{R}$). An $n$\textbf{-element $\max$ network} is 
	denoted by $\Theta_{\max_n}$, and functions analogously. These networks are 
	described in \cite{FerlezAReNAssuredReLU2020}.
\end{definition}

With \Cref{def:functional_composition} through \Cref{def:n-element_minmax_NN} 
in hand, it is now possible for us to define TLL NNs in just the same way as 
\cite{FerlezBoundingComplexityFormally2020}. We likewise proceed to first 
define a \emph{scalar (or real-valued)} TLL NN; the structure of such a scalar 
TLL NN is illustrated in \Cref{fig:tll_arch}. Then we extend this notion to a 
\emph{multi-output (or vector-valued)} TLL NN.
\begin{definition}[Scalar TLL NN {\cite{FerlezBoundingComplexityFormally2020}}]
\label{def:scalar_tll}
A NN that maps $\mathbb{R}^n \rightarrow \mathbb{R}$ is said to be \textbf{TLL 
NN of size} $(N,M)$ if the size of its parameter list $\Xi_{\scriptscriptstyle 
N,M}$ can be characterized entirely by integers $N$ and $M$ as follows.
\begin{equation}
	\Xi_{N,M} \negthinspace \triangleq  \negthinspace
		\Theta_{\max_M} \negthinspace\negthinspace
	\circ \negthinspace
		\big(
			(\negthinspace\Theta_{\min_N} \negthinspace \circ \Theta_{S_1}\negthinspace) \negthinspace
			\parallel \negthinspace {\scriptstyle \dots} \negthinspace \parallel \negthinspace
			(\negthinspace\Theta_{\min_N} \negthinspace \circ \Theta_{S_M}\negthinspace)
		\big) \negthinspace
	\circ 
		\Theta_{\ell}
\end{equation}
where

\begin{itemize}
	\item $\Theta_\ell \triangleq ((W_\ell, b_\ell))$;

	\item  each $\Theta_{S_j}$ has the form $\Theta_{S_j} = \big( S_j,  
		\mathbf{0}_{N,1} \big)$; and

	\item $S_j = \left[ \begin{smallmatrix} {\llbracket I_N 
		\rrbracket_{[\iota_1, 
		:]}}\negthickspace\negthickspace\negthickspace^{^{\scriptscriptstyle\text{T}}} 
		& \; \dots \; & {\llbracket I_N \rrbracket_{[\iota_N, 
		:]}}\negthickspace\negthickspace\negthickspace^{^{\scriptscriptstyle\text{T}}} 
		\end{smallmatrix} \right]^\text{T}$ for some length-$N$ sequence 
		$\{\iota_k\} \subseteq \{1, \dots, N\}$ where $I_N$ is the $(N \times 
		N)$ identity matrix. 
\end{itemize}


The linear functions implemented by the mapping  
$\mathscr{L}_i\subarg{W_\ell,b_\ell}$ for $i = 1, \dots, N$ will be referred to 
as the \textbf{local linear functions} of $\Xi_{N,M}$; we assume for simplicity 
that these linear functions are unique. The matrices $\{ S_j | j = 1, \dots, 
M\}$ will be referred to as the \textbf{selector matrices} of $\Xi_{N,M}$. Each 
set $s_j \triangleq \{ k \in \{1, \dots, N\} | \exists \iota \in \{1, \dots, 
N\}. \llbracket S_j \rrbracket_{\iota,k} = 1 \}$ is said to be the  
\textbf{selector set of} $S_j$.
\end{definition}

\begin{definition}[Multi-output TLL NN {\cite{FerlezBoundingComplexityFormally2020}}]
\label{def:multi-output_tll}
	A NN that maps $\mathbb{R}^n \rightarrow \mathbb{R}^m$ is said to be a 
	\textbf{multi-output TLL NN of size} $(N,M)$ if its parameter list 
	$\Xi_{\scriptscriptstyle N,M}^{(m)}$ can be written as 
	\begin{equation}
	\label{eq:multi_out_tll}
		\Xi_{\scriptscriptstyle N,M}^{(m)} = \Xi_{\scriptscriptstyle N, M}^1 
			\parallel
			\dots
			\parallel
			\Xi_{\scriptscriptstyle N, M}^m
	\end{equation}
	for $m$ equally-sized scalar TLL NNs, $\Xi_{\scriptscriptstyle N, M}^1,  
	\dots, \Xi_{\scriptscriptstyle N, M}^m$; these scalar TLLs will be referred 
	to as the \textbf{(output) components of} $\Xi_{\scriptscriptstyle 
	N,M}^{(m)}$.
\end{definition}



\subsection{Hyperplanes and Hyperplane Arrangements} 
\label{sub:hyperplanes_and_hyperplane_arrangements}
Here  we review notation for hyperplanes and hyperplane arrangements; these 
results will be important in the developemnt of \myalg. 
\cite{StanleyIntroductionHyperplaneArrangements} is the main reference for this 
section.
\begin{definition}[Hyperplanes and Half-spaces]
\label{def:hyperplane}
	Let $\ell : \mathbb{R}^n \rightarrow \mathbb{R}$ be an affine map. Then 
	define:
	\begin{equation}
		H^{q}_{\ell} \triangleq 
		\begin{cases}
			\{x | \ell(x) < 0 \} & q = -1 \\
			\{x | \ell(x) > 0 \} & q = +1 \\
			\{x | \ell(x) = 0 \} & q = 0.
		\end{cases}
	\end{equation}
	We say that $H^{0}_\ell$ is the \textbf{hyperplane defined by} $\ell$ 
	\textbf{in dimension }$n$, and $H^{-1}_\ell$ and $H^{+1}_\ell$ are the 
	\textbf{negative and positive half-spaces defined by} $\ell$, respectively.
\end{definition}



\begin{definition}[Hyperplane Arrangement]
	Let $\mathcal{L}$ be a set of affine functions where each $\ell \in 
	\mathcal{L} : \mathbb{R}^n \rightarrow \mathbb{R}$. Then $\{ H^{0}_{\ell} | 
	\ell \in \mathcal{L} \}$ is an \textbf{arrangement of hyperplanes in 
	dimension} $n$.
\end{definition}

\begin{definition}[Region of a Hyperplane Arrangement]
\label{def:hyperplane_region}
	Let $\mathcal{H}$ be an arrangement of $N$ hyperplanes in dimension $n$ 
	defined by a set of affine functions, $\mathcal{L}$. Then a non-empty open 
	subset $R \subseteq \mathbb{R}^n$ is said to be a 
	\textbf{region of} $\mathcal{H}$ if there is an indexing function 
	$\mathfrak{s} : \mathcal{L} \rightarrow \{-1, 0, +1\}$ such that
	$R = \bigcap_{\ell \in \mathcal{L}} H^{\mathfrak{s}(\ell)}_\ell$; $R$ is 
	said to be \textbf{$n$-dimensional or full-dimensional} if it is non-empty 
	and described by an indexing function $\mathfrak{s}(\ell) \in \{-1, +1\}$ 
	for all $\ell \in  \mathcal{L}$.
\end{definition}




\begin{theorem}[\cite{StanleyIntroductionHyperplaneArrangements}]
\label{thm:arrangement_regions_bound}
	Let $\mathcal{H}$ be an arrangement of $N$ hyperplanes in dimension $n$. 
	Then $|\mathcal{R}_\mathcal{H}|$ is at most $\sum_{k=0}^n {N \choose k}$.
\end{theorem}

\begin{remark}
	Note that for a fixed dimension, $n$, the bound $\sum_{k=0}^n \negthinspace 
	{N \choose k}$ grows like $O(N^n\negthinspace/n!)$, i.e. sub-exponentially 
	in $N$.
\end{remark}

%% file: problem.tex
\section{Problem Formulation} 
\label{sec:problem}
The essence of \myalg~is its focus on verifying TLL NNs with respect to 
box-like output constraints. Formally, \myalg~considers only verification 
problems of the following form (stated using notation from 
\Cref{sec:preliminaries}).
\begin{problem}
\label{prob:tll_verification_multi}
	Let $\nn\subarg{\Xi_{\scriptscriptstyle N,M}^{(m)}}: \mathbb{R}^n 
	\rightarrow \mathbb{R}^m$ be a multi-output TLL NN. Also, let:

	\begin{itemize}
		\item $P_X \subset \mathbb{R}^n$ be a closed, convex polytope 
			specified by the intersection of $\mathsf{N}_X$ half-spaces, i.e. 
			$P_X \triangleq \cap_{i=1}^{\mathsf{N}_X}  \{x:\ell_{X,i}(x) \leq  
			0\}$ where each $\ell_{X,i} : \mathbb{R}^n \rightarrow  \mathbb{R}$ 
			is affine; and

		\item $P_Y \subset \mathbb{R}^m$ be closed hyper-rectangle, i.e. 
			$P_Y \triangleq \prod_{k=1}^{m} [a_k, b_k]$ with $-\infty \leq a_k 
			\leq b_k \leq \infty$ for each $k$.
	\end{itemize}
	Then the verification problem is to decide whether the following formula is 
	true:
	\begin{equation}
	\label{eq:verification_formula}
		\forall x \in P_X \subset \mathbb{R}^n . \big( \nn\subarg{\Xi_{\scriptscriptstyle N,M}^{(m)}}\negthinspace(x) \in P_Y \subset \mathbb{R}^m \big).
	\end{equation}
	If \eqref{eq:verification_formula} is true, the problem is \textbf{SAT}; 
	otherwise, it is \textbf{UNSAT}.
\end{problem}

Note that the properties (and their interpretations) in 
\Cref{prob:tll_verification_multi} \emph{are dual to the usual convention}; it 
is more typical in the literature to associate ``unsafe'' outputs with a 
closed, convex polytope, and then the \emph{existence} of such unsafe outputs 
is denoted by  \textbf{UNSAT} (see \cite{TranStarBasedReachabilityAnalysis2019} 
for example). However, we chose this formulation for 
\Cref{prob:tll_verification_multi} because it is the one adopted by 
\cite{FerlezBoundingComplexityFormally2020}, and because it is more suited to  
NN reachability computations, one of the motivating applications of \myalg. 
Indeed, to verify a property like \eqref{eq:verification_formula}, the typical 
dual formulation of \Cref{prob:tll_verification_multi} would require $2\cdot  
m$ verifier calls, assuming unbounded polytopes are verifiable (and then the 
verification would only be respect to the interior of $P_Y$). Of course this 
choice comes with a trade-off: \myalg, which directly solves 
\Cref{prob:tll_verification_multi}, requires adaptation to verify the dual 
property of \Cref{prob:tll_verification_multi}; we return to this briefly at 
the end of this section, but it is ultimately left for future work.

In the case of \myalg, there is another important reason to consider the stated 
formulation of \Cref{prob:tll_verification_multi}: both the output property 
$P_Y$ and the NN $\Xi_{\scriptscriptstyle N,M}^{(m)}$ have an essentially  
\emph{component-wise nature} (see also \Cref{def:multi-output_tll}). As a  
result, a component-wise treatment of \Cref{prob:tll_verification_multi} 
greatly facilitates the development and operation of \myalg. To this end, we 
will find it convenient in the sequel to consider the following two 
verification problems; each is specified for a \emph{scalar TLL NNs and a 
single real-valued output property.} Moreover, we cast them in terms of the 
negation of the analogous formula derived from 
\Cref{prob:tll_verification_multi}; the reasons for this will become clear in 
\Cref{sec:algorithm}.

\begin{subproblem}[Scalar Upper Bound]
\label{prob:scalar_upper_bound}
Let $\nn\subarg{\Xi_{\scriptscriptstyle N,M}}: \mathbb{R}^n \rightarrow 
\mathbb{R}$ be a \textbf{scalar TLL NN}, and let $P_X \triangleq 
\cap_{i=1}^{\mathsf{N}_X}  \{x:\ell_{X,i}(x) \leq  0\}$ be a closed convex 
polytope  as in \Cref{prob:tll_verification_multi}.

Then the \textbf{scalar upper bound verification problem for } $a \in  
\mathbb{R}$ is to decide whether the following formula is true:
\begin{equation}
\label{eq:scalar_upper_bound_formula}
	\exists x \in P_X \subset \mathbb{R}^n . \big( \nn\subarg{\Xi_{\scriptscriptstyle N,M}}\negthinspace(x) > b \big).
\end{equation}
If \eqref{eq:scalar_upper_bound_formula} is true, the problem is  
\textbf{UNSAT}; otherwise, it is \textbf{SAT}.
\end{subproblem}
\begin{subproblem}[Scalar lower Bound]
\label{prob:scalar_lower_bound}
Let $\nn\subarg{\Xi_{\scriptscriptstyle N, M}}: \mathbb{R}^n \rightarrow  
\mathbb{R}$  and $P_X$ be as in \Cref{prob:scalar_upper_bound}.

Then the \textbf{scalar lower bound verification problem for } $b \in  
\mathbb{R}$ is to decide whether the following formula is true:
\begin{equation}
\label{eq:scalar_lower_bound_formula}
	\exists x \in P_X \subset \mathbb{R}^n . \big( \nn\subarg{\Xi_{\scriptscriptstyle N,M}}\negthinspace(x) < a \big).
\end{equation}
If \eqref{eq:scalar_lower_bound_formula} is true, the problem is  
\textbf{UNSAT}; otherwise, it is \textbf{SAT}.
\end{subproblem}

Thus, note that the formulation of \Cref{prob:tll_verification_multi} is such 
that it can be verified by evaluating a boolean formula that contains only 
instances of \Cref{prob:scalar_upper_bound} and \Cref{prob:scalar_lower_bound}. 
That is, the following formula has the same truth value as 
\eqref{eq:verification_formula}:
\begin{multline}\label{eq:prob1_as_1ab}
	\bigwedge_{k=1}^m \Bigg(
		\neg \Big(\exists x \in P_X \subset \mathbb{R}^n . \big( \nn\subarg{\Xi_{\scriptscriptstyle N,M}^k}\negthinspace(x) < a_k \big) \Big) \quad \wedge \\[-10pt]
		\neg \Big( \exists x \in P_X \subset \mathbb{R}^n . \big( \nn\subarg{\Xi_{\scriptscriptstyle N,M}^k}\negthinspace(x) < b_k \big) \Big)
	\Bigg).
\end{multline}
We reiterate, however, that the same is not true of the \emph{dual} property to 
\Cref{prob:tll_verification_multi}. Consequently, \myalg~requires modification 
to verify such properties; this is a more or less straightforward procedure,  
but we defer this to future work, as noted above.



%% file: algorithm.tex

\section{\myalg: Theory} 
\label{sec:algorithm}
In this section, we develop the theoretical underpinnings of \myalg. As noted 
in \Cref{sec:problem}, the essential insight of our algorithm is captured by 
our solutions to problems \Cref{prob:scalar_upper_bound} and 
\Cref{prob:scalar_lower_bound}. Thus, this section is organized primarily 
around solving sub-problems of these forms; at the end of this section, we will 
show how to combine these results into a verification algorithm for 
\Cref{prob:tll_verification_multi}, and then we will analyze the overall 
computational complexity of \myalg.

\subsection{Verifying {\Cref{prob:scalar_upper_bound}}} 
\label{sub:verifying_upper_bound}
\Cref{prob:scalar_upper_bound}, as stated above, regards the TLL NN to be 
verified merely as a map from inputs to outputs; this is the behavior that we 
wish to verify, after all. However, this point of view obscures the 
considerable semantic structure intrinsic to the neurons in a TLL NN. In 
particular, recall that $\nn\subarg{\Xi_{\scriptscriptstyle N, M}}$ implements 
the following function, which was derived from the Two-Level Lattice 
representation of CPWAs -- see \Cref{sub:two_layer_lattice_neural_networks} and 
\cite{FerlezAReNAssuredReLU2020, TarelaRegionConfigurationsRealizability1999}:
\begin{equation}
\label{eq:tll_implementation}
	\nn\subarg{\Xi_{\scriptscriptstyle N, M}} (x) = \max_{1 \leq j \leq M} 
	\min_{i \in s_j} \mathscr{L}_i\subarg{W_\ell, b_\ell}(x).
\end{equation}
In \eqref{eq:tll_implementation}, the sets $s_j$ are the \emph{selector sets} 
of $\nn\subarg{\Xi_{\scriptscriptstyle N, M}}$ and the 
$\mathscr{L}_i\subarg{W_\ell,b_\ell}$ are the \emph{local linear functions} of 
$\nn\subarg{\Xi_{\scriptscriptstyle N, M}}$; both terminologies are formally 
defined in \Cref{def:scalar_tll}. Upon substituting 
\eqref{eq:tll_implementation} into \eqref{eq:scalar_upper_bound_formula}, we 
obtain the following, far more helpful representation of the property expressed 
in \Cref{prob:scalar_upper_bound}:
\begin{equation}
\label{eq:tll_substitution_upper}
	\exists x \in P_X \subset \mathbb{R}^n . \big( 
		\max_{1 \leq j \leq M} \min_{i \in s_j} \mathscr{L}_i\subarg{W_\ell, b_\ell}(x) 
		> b 
	\big).
\end{equation}
Literally, \eqref{eq:tll_substitution_upper} compares the output property of 
interest, $b \in \mathbb{R}$, with a combination of \emph{real-valued} 
$\max/\min$ operations applied to scalar affine functions. Crucially,  that 
comparison is made using the usual order relation on $\mathbb{R}$, $\geq$,  
which is exactly the same order relation upon which the $\max$ and $\min$ 
operations are based.

Thus, it is possible to simplify \eqref{eq:tll_substitution_upper} as follows. 
First note that the result of the $\max$ operation in 
\eqref{eq:tll_substitution_upper} can exceed $b$ on $P_X$ if and only if:
\begin{equation}
\label{eq:dist_inequality_to_mins}
	\exists x \in P_X . \exists
	j \in \{1, \dots, M\} . \Big( \min_{i \in s_j} 
\mathscr{L}(W_\ell, b_\ell)(x) > b \Big).
\end{equation}
In turn, the result of any one of the $\min$ operations in 
\eqref{eq:tll_substitution_upper} can exceed $b$ on $P_X$, and hence make 
\eqref{eq:dist_inequality_to_mins} true, if and only if
\begin{equation}
\label{eq:dist_inequality_to_lins}
	\exists x \in P_X . \forall i \in s_j . \Big(
		\mathscr{L}_i\subarg{W_\ell,b_\ell}(x)
		> b
	\Big).
\end{equation}
In particular, \eqref{eq:dist_inequality_to_lins} is actually an intersection 
of \emph{half spaces}, some open and some closed: the open half spaces come 
from local linear functions that violate the property; and the closed 
half-spaces come from the input property, $P_X$ (see 
\Cref{prob:tll_verification_multi}). Moreover, there are at most $M$ such 
intersections of relevance to \Cref{prob:scalar_upper_bound}: one for each of 
the $j = 1, \dots, M$ such $\min$ operations present in 
\eqref{eq:tll_substitution_upper}. Finally, note that linear feasibility 
problems consisting entirely of \emph{non-strict} inequality constraints are 
easy to solve: this suggests that we should first amend the $>b$ inequality 
with $\geq b$ before proceeding.

Formally, these ideas are captured in the following proposition.
\begin{proposition}
\label{prop:solving_upper_bound}
	Consider an instance of \Cref{prob:scalar_upper_bound}. Then that instance 
	is \textbf{UNSAT} if and only if the set:
	\begin{equation}\label{eq:prop_prob1b_main_formula}
		\bigcup_{j = 1, \dots, M}
		\left(
			\bigcap_{i \in s_j} \{x \negthinspace: \mathscr{L}_i\subarg{W_\ell,b_\ell} > b \}
			\cap
			\bigcap_{i=1}^{\mathsf{N}_X} \{x:\negthinspace \ell_{X,i}(x) \leq  0\} \negthinspace
		\right) \neq \emptyset.
	\end{equation}
	Or equivalently, if for at least one of the $j = 1, \dots, M$, the linear 
	feasibility problem specified by the constraints
	\begin{equation}
		F_j \triangleq
		\begin{cases}
			\mathscr{L}_{i_1}\subarg{W_\ell,b_\ell}(x) \geq b \\
			\vdots \\
			\mathscr{L}_{i_{|s_j|}}\subarg{W_\ell,b_\ell}(x) \geq b
		\end{cases}
		\bigwedge \;\;
		\begin{cases}
			\ell_{X,1}(x) \leq 0 \\
			\vdots \\
			\ell_{X,\mathsf{N}_X}(x) \leq 0
		\end{cases}
	\end{equation}
	is feasible, \textbf{and} one of the following conditions is true: 
	\begin{itemize}
		\item it has non-empty interior; or

		\item there is a feasible point that lies only on some subset of the 
			$\{\ell_{X,i} : i = 1, \dots, \mathsf{N}_X\}$.
	\end{itemize}
\end{proposition}
\begin{proof}
	The first claim follows immediately from the manipulations described in 
	\eqref{eq:dist_inequality_to_mins} and \eqref{eq:dist_inequality_to_lins}. 
	The second claim merely exhausts the possibilities for how the constraints 
	$\mathcal{L}_i\subarg{W_\ell,b_\ell}(x) > b$ can participate in a feasible 
	set for the linear program given by $F_j$.
\end{proof}
\begin{remark}
\label{rem:lattice_to_set_upper_bound}
	The conclusion of \Cref{prop:solving_upper_bound} also has the following 
	important interpretation: the $\exists x . ( \dots > b)$ property can be 
	seen to ``distribute across'' the $\max/\min$ operations in 
	\eqref{eq:tll_substitution_upper}, and upon doing so, it converts the 
	lattice operation $\max$ into \emph{set union} and the lattice operation 
	$\min$ into \emph{set intersection}. Furthermore, since a TLL NN is 
	constructed from two levels of lattice operations applied to affine 
	functions, the innermost lattice operation of $\min$  is converted into a 
	set intersection of half-spaces --- i.e., a linear feasibility problem.
\end{remark}

Of course \Cref{prop:solving_upper_bound} also suggests a natural and obvious 
algorithm to verify an instance of \Cref{prob:scalar_upper_bound}. The 
pseudocode for this algorithm appears as the function \texttt{verifyScalarUB} 
in \Cref{alg:solve_problem_upper_bound}, and its correctness follows directly  
from \Cref{prop:solving_upper_bound}. In particular, \texttt{verifyScalarUB} 
simply evaluates the feasibility of each set of constraints $F_j, j = 1,  
\dots, M$ in turn, until either a feasible problem is found or the list is  
exhausted. Then for each such feasible $F_j$, \texttt{verifyScalarUB} attempts 
to find an interior point of the feasible set to reconcile it with the desired 
inequalities in \eqref{eq:prop_prob1b_main_formula}; failing that, it searches 
for a vertex of the feasible set where no output property constraints are 
active. In practice, these operations can be combined by operating on the 
feasible point returned by the original feasibility program: an LP can be used 
to maximize the value of each constraint activate in order to explore  adjacent 
vertices. Note further that \texttt{verifyScalarUB} may not need to execute all 
$M$ possible linear  programs for properties that are \textbf{UNSAT}: it can 
terminate early on the first ``satisfied'' linear program found.

\IncMargin{0.5em}
\begin{algorithm}[!t]


\SetKwData{false}{False}
\SetKwData{true}{True}
\SetKwData{feas}{Feasible}

\SetKwData{constraints}{constraints}
\SetKwData{status}{status}
\SetKwData{sol}{sol}
\SetKwData{sollist}{solnList}
\SetKwData{actconstr}{actConstr}

\SetKwFunction{append}{append}
\SetKwFunction{verifyScalarLB}{verifyScalarUB}
\SetKwFunction{solvefeas}{SolveLinFeas}
\SetKwFunction{findint}{FindInt}
\SetKwFunction{all}{all}

\SetKw{Break}{break}
\SetKw{not}{not}
\SetKw{foriter}{for}
\SetKw{OR}{or} 

\SetKwInOut{Input}{input}
\SetKwInOut{Output}{output}

\Input{
\hspace{3pt}$b \in \mathbb{R}$, a upper bound to verify \\
\hspace{3pt}$\Xi_{\scriptscriptstyle N,M}$, parameters of a TLL NN to verify \\
\hspace{3pt}$L_X = \{\ell_{X,1}, \dots, \ell_{X,\mathsf{N}_X}\}$, affine functions \\
\hspace{15pt}specifying an input constraint polytope, $P_X$\\
}
\Output{
\hspace{3pt}Boolean ($\mathtt{True = SAT}; \mathtt{False = UNSAT}$)
}
\BlankLine
\SetKwProg{Fn}{function}{}{end}%
\Fn{\verifyScalarLB{$b$, $\Xi_{\scriptscriptstyle N,M}$, $L_X$}}{

	\For{$j = 1, \dots, M$}{

		\constraints $\leftarrow$ $[\big(\mathscr{L}_i\subarg{W_\ell,b_\ell}(x) \geq b\big) $ \foriter $i \in s_j]$

		\constraints.\append{$[ \ell_{X,1}(x)\leq~0, \dots, \ell_{X,N_X}(x)\leq~0 ]$}

		(\sol, \status) $\leftarrow$ \solvefeas{\constraints}


		\If{\status $==$ \feas}{
			\If{\all{$[ \mathscr{L}_i\subarg{W_\ell,b_\ell}($ \sol $) > b$  \foriter  $i = 1\dots N]$} \\
			\OR \findint{\constraints} == \true
			}{
				\Return \false
			}


		}
	}
	\Return \true
}
\caption{\texttt{verifyScalarUB}; i.e., solve \Cref{prob:scalar_upper_bound}}
\label{alg:solve_problem_upper_bound}
\end{algorithm}
\DecMargin{0.5em}


\subsection{Verifying {\Cref{prob:scalar_lower_bound}}} 
\label{sub:verifying_lower_bound}
Naturally, we start our consideration of \Cref{prob:scalar_lower_bound} in very 
much the same way as \Cref{prob:scalar_upper_bound}. However, given that 
\Cref{prob:scalar_lower_bound} and \Cref{prob:scalar_upper_bound} are in some 
sense \emph{dual}, the result is not nearly as convenient. In particular, 
substituting \eqref{eq:tll_implementation} into 
\eqref{eq:scalar_lower_bound_formula}, and attempting carry out the same 
sequence of manipulations that led to \Cref{prop:solving_upper_bound} results 
in the following formula:
\begin{equation}
\label{eq:distribute_set_lower}
 	\bigcap_{j = 1, \dots, M} \left(
 		\bigcup_{i \in s_j} \{x : \mathscr{L}_i\subarg{W_\ell,b_\ell}(x) < a \}
 		\cap
 		\bigcap_{i=1}^{\mathsf{N}_X} \{x:\negthinspace \ell_{X,i}(x) \leq  0\} 
 	\right) \neq \emptyset,
 \end{equation} 
which has the same truth value as formula 
\eqref{eq:scalar_lower_bound_formula}. Unfortunately, 
\eqref{eq:distribute_set_lower} is not nearly as useful as the result in 
\Cref{prop:solving_upper_bound}: under the ``dual'' output constraint $< a$, 
set intersection and set union are switched relative to 
\Cref{prop:solving_upper_bound}. As a consequence, 
\eqref{eq:distribute_set_lower} is not itself a direct formulation in terms of 
intersections of half-spaces --- i.e., linear feasibility problems.

Nevertheless, rearranging \eqref{eq:distribute_set_lower} into the  
union-of-half-space-int-\\ ersections form of \Cref{prop:solving_upper_bound} 
is possible and profitable. Using basic set operations, it is possible to 
rewrite \eqref{eq:distribute_set_lower} in a union-of-intersections form as  
follows (the set intersection $P_X = \cap_{i = 1}^{\mathsf{N}_X} \{x: 
\ell_{X,i}(x) \leq 0\}$ is moved outside the outer union for convenience):
\begin{equation}
\label{eq:lower_bound_union_of_intersections}
	P_X \cap
	\bigcup_{
		\begin{smallmatrix}
			(i_1, \dots, i_M) \\
			\in s_1 \times \dots \times s_M
		\end{smallmatrix}
	}
	\bigcap_{k=1, \dots, M}
	\{x : \mathscr{L}_{i_k}\subarg{W_\ell,b_\ell}(x) < a\} \neq \emptyset.
\end{equation}
\noindent By construction, \eqref{eq:lower_bound_union_of_intersections} again 
has the same truth value as  \eqref{eq:scalar_lower_bound_formula}, but it is 
now in the desired form. In particular, it is verifiable by evaluating a finite 
number of half-space intersections much like \Cref{prop:solving_upper_bound}.

Unfortunately, as a result of this rearrangement, the total number of mutual 
half-space intersections -- or intersection ``terms'' -- has grown from $M$ to 
$ \prod_{j=1}^{M} |s_j|$, where $|s_j|$ is the cardinality of the selector set, 
$s_j$. In particular, this number can easily exceed $M$: for example, if each 
of the $s_j$ has exactly two elements, then there are $2^M$ total mutual 
intersection terms. Thus, verifying 
\eqref{eq:lower_bound_union_of_intersections} in its current form would appear 
to require (in the worst case) exponentially more linear feasibility programs 
than the verifier we proposed for \Cref{prob:scalar_upper_bound}. This 
situation is not only non-ideal in terms  of run-time: it would also seem to 
contradict \cite{FerlezBoundingComplexityFormally2020}, which describes an 
algorithm with polynomial-time complexity in $N$ and $M$ --- and that algorithm 
is after all applicable to \emph{more general} output properties.

Fortunately, however, there is one aspect not emphasized in this analysis so 
far: these intersection terms consist of \emph{half-spaces}, and moreover, each 
of the half-spaces therein is specified by a hyperplane chosen from among a 
single, common group of $N$ hyperplanes. This will ultimately allow us to 
identify each non-empty intersection term in 
\eqref{eq:lower_bound_union_of_intersections} with a \emph{full,  
$n$-dimensional region} from this hyperplane arrangement, and by 
\Cref{thm:arrangement_regions_bound} in 
\Cref{sub:hyperplanes_and_hyperplane_arrangements}, there are at most 
$O(N^n/n!)$ such regions. Effectively, then, the \emph{geometry} of this 
hyperplane arrangement (with $N$ hyperplanes in dimension $n$) prevents 
exponential growth in the number intersection terms relevant to the truth of 
\eqref{eq:lower_bound_union_of_intersections}: indeed, the polynomial growth, 
$O(N^n/n!)$, means that many of those intersection terms cannot correspond to 
valid regions in the arrangement\footnote{Of course these results apply when 
$n$ is fixed; see also \cite{FerlezBoundingComplexityFormally2020}, and the  
comments therein pertaining to NN verification encodings of 3-SAT problems  
\cite{KatzReluplexEfficientSMT2017a}.}.

In particular, consider the following set of affine functions, which in turn 
defines an arrangement of hyperplanes in $\mathbb{R}^n$:
\begin{equation}
\label{eq:prob1b_hyperplane_arrangement}
	l_i^a \triangleq \mathscr{L}_i\subarg{W_\ell,b_\ell} - a.
\end{equation}
Let $\mathfrak{h}_a \triangleq \cup_{i = 1}^N \{ H^0_{l_i^b} \}$ denote the 
corresponding arrangement. Now, consider any index $(i_1, \dots, i_M) \in s_1 
\times \dots \times s_M$ specifying an intersection term in 
\eqref{eq:lower_bound_union_of_intersections}, and  suppose without loss of 
generality that $i_1, \dots, i_K$ are the only unique indices therein (an 
assumption we carry forward). Then using the notation $H^{-1}$ introduced in 
\Cref{def:hyperplane}, it is straightforward to write:
\begin{equation}\label{eq:prob1b_hyperplane_subset}
	\bigcap_{k=1, \dots, M} \hspace{-5pt}
	\{x : \mathscr{L}_{i_k}\subarg{W_\ell,b_\ell}(x) < a \}
	=
	\bigcap_{i = 1}^K 
	H^{-1}_{l_i^a},
\end{equation}
and where $l_i^a$ is as defined in \eqref{eq:prob1b_hyperplane_arrangement}.  
As a consequence, we conclude:
\begin{multline}\label{eq:sub_region_nonempty}
	\bigcap_{k=1, \dots, M} \hspace{-5pt}
	\{x : \mathscr{L}_{i_k}\subarg{W_\ell,b_\ell}(x) < a \} \neq \emptyset \\
	\Leftrightarrow
	\exists (\mathfrak{s}^\prime : \{l_i^a: i = 1, \dots, K\} \rightarrow \{-1\}) . \Big(
		\bigcap_{i = 1}^K H^{\mathfrak{s}^\prime(l_i^a)}_{l_i^a} \neq \emptyset \Big).
\end{multline}
Although it seems unnecessary to introduce the function $\mathfrak{s}^\prime$, 
this notation directly connects \eqref{eq:prob1b_hyperplane_subset} to 
full-dimensional regions of the arrangement $\mathfrak{h}_a$. Indeed, it states 
that the intersection term of interest is non-empty if and only if there is a 
full-dimensional region in the hyperplane arrangement whose index function 
$\mathfrak{s} : \{l_i^a: i = 1, \dots N\} \rightarrow  \{-1,+1\}$ agrees with 
one of the partial indexing functions $\mathfrak{s}^\prime$ described in 
\eqref{eq:sub_region_nonempty}. \emph{More simply, said intersection term is 
non-empty if and only if it contains a  
full-dimensional region from the arrangement $\mathfrak{h}_a$; such a region 
can be said to ``witness'' the non-emptiness  of the intersection term.} This 
idea is illustrated in \Cref{fig:prob1b_illustration}.

\begin{figure}[t]
	\centering %
	\includegraphics[width=0.4\textwidth]{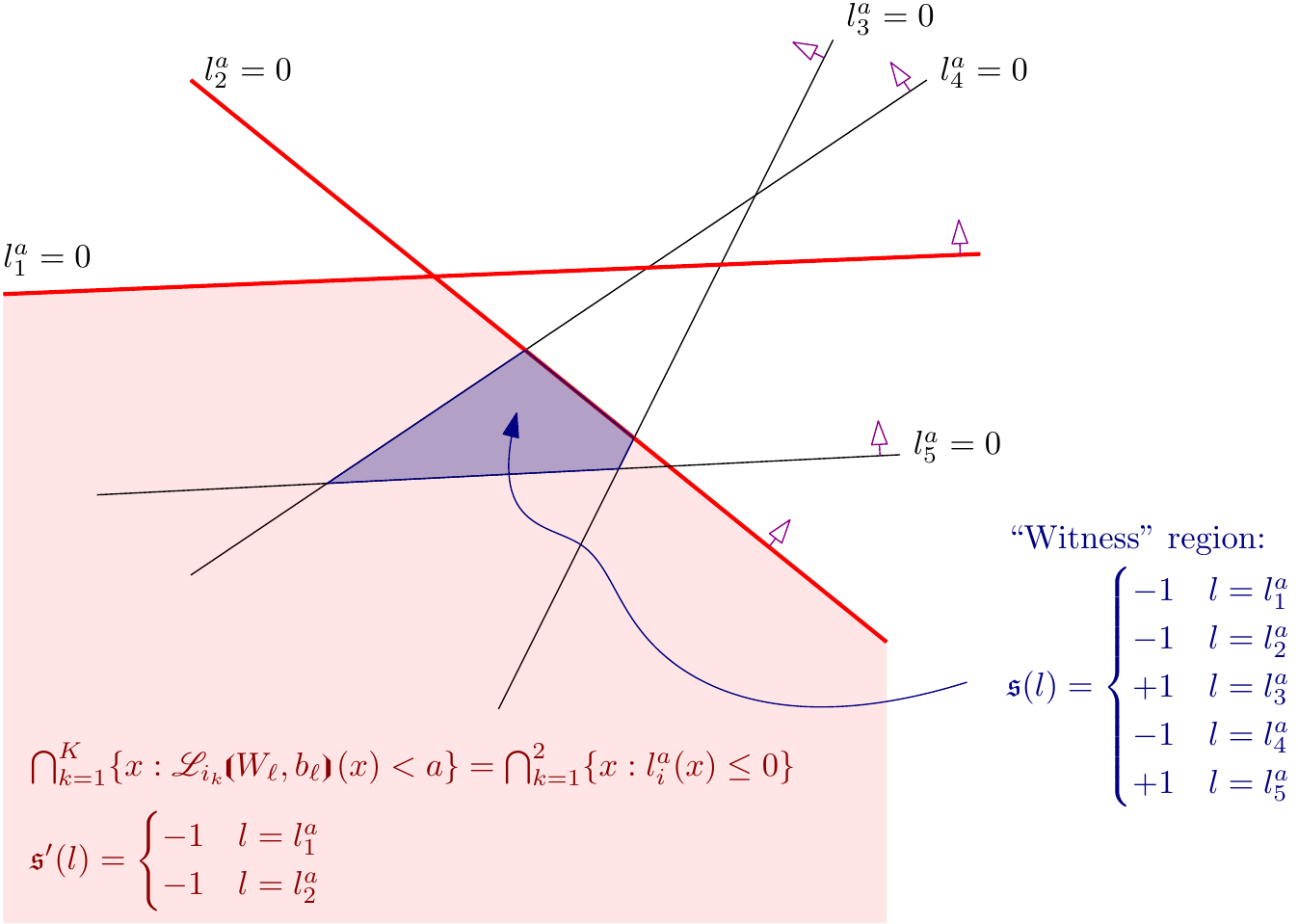} %
	\caption{Illustration of identifying a non-empty intersection term from \eqref{eq:lower_bound_union_of_intersections} with a full-dimensional region from the hyperplane arrangement $\mathfrak{h}_a$; positive half-spaces are indicated with blue arrows. Intersection term and defining half-spaces shown in red; ``witness'' region from $\mathfrak{h}_a$ shown in blue. Input constraints, $P_X$, omitted for clarity.} %
	\label{fig:prob1b_illustration} %
\end{figure}

Formally, we have the following proposition.
\begin{proposition}
\label{prop:solving_lower_bound}
	Consider an instance of \Cref{prob:scalar_lower_bound}. Then that instance 
	is \textbf{UNSAT} if and only if the set:
	\begin{equation}
		P_X \cap
		\bigcup_{
			\begin{smallmatrix}
				(i_1, \dots, i_M) \\
				\in s_1 \times \dots \times s_M
			\end{smallmatrix}
		}
		\bigcap_{k=1, \dots, M}
		\{x : \mathscr{L}_{i_k}\subarg{W_\ell,b_\ell}(x) < a \} \neq \emptyset.
	\end{equation}
	And this is the case if and only if there exists an index $(i_1, \dots, 
	i_M) \in s_1 \times \dots \times s_M$ with distinct elements denoted by 
	$j_1,  \dots, j_K \in \{1, \dots, M\}$ such that the following holds:
	\begin{itemize}
		\item there exists a region $R$ in $\mathfrak{h}_a$, specified by  
			$\mathfrak{s}_R: \{l_i^a : i = 1, \dots, N\} \rightarrow \{-1,0,  
			+1\}$, such that:
			\begin{equation}
				\mathfrak{s}_R(l^a_{j_k}) = -1 \text{ for all } k = 1, \dots, K
			\end{equation}
			and
			\begin{equation}
				R \cap P_X \neq \emptyset.
			\end{equation}
	\end{itemize}
	If such a region $R$ exists, then it is said to \textbf{witness} the 
	non-emptiness of the corresponding intersection term with the index $(i_1, 
	\dots, i_M)$.
\end{proposition}
\begin{proof}
	The proof follows from the above manipulations. 
\end{proof}

\Cref{prop:solving_lower_bound} establishes a crucial identification between  
full-dimensional regions in a hyperplane arrangement and the non-empty 
intersection terms in \eqref{eq:lower_bound_union_of_intersections}, a 
verification formula equivalent to the satisfiability of 
\Cref{prob:scalar_lower_bound}. However, it is still framed in terms of 
individual indices of the form $(i_1, \dots, i_M)$, which are too numerous to 
enumerate for reasons noted above. Thus, converting 
\Cref{prop:solving_lower_bound} into a practical and fast algorithm to solve 
\Cref{prob:scalar_lower_bound} entails one final step: efficiently evaluating a 
\emph{full-dimensional} region in $\mathfrak{h}_a$ to determine if it matches 
\emph{any} index of the form form $(i_1, \dots, i_M) \in s_1 \times \dots 
\times s_M$. This will finally lead to \myalg's algorithm to verify an instance 
\Cref{prob:scalar_lower_bound} by enumerating the regions of $\mathfrak{h}_a$ 
instead of enumerating all of the indices in $s_1 \times \dots \times s_M$.

Predictably, \myalg~essentially takes a greedy approach to this problem. In 
particular, consider a full-dimensional region, $R \subset \mathbb{R}^n$,  of 
the hyperplane arrangement $\mathfrak{h}_a$, and suppose that $R$ is specified 
by the index function (see \Cref{def:hyperplane_region}):
\begin{equation}
	\mathfrak{s}_R : \{l_i^a | i = 1, \dots N\} \rightarrow \{-1,+1\}.
\end{equation}
According to \Cref{prop:solving_lower_bound}, $\mathfrak{s}_R$ will be a 
witness to a violation of \Cref{prob:scalar_lower_bound} if each of its 
negative hyperplanes (those assigned $-1$ by $\mathfrak{s}_R$) can be matched 
to \emph{any} element of one of the selector sets, $s_j$ $j=1,\dots,M$. Thus, 
to establish whether $\mathfrak{s}_R$ corresponds to a non-empty intersection 
term, we can iterate over its negative hyperplanes, checking each one for 
membership in any one of the $s_j$. This iteration proceeds as long as each 
negative hyperplane is found to be an element of \emph{some} $s_j$. If all 
negative hyperplanes of $R$ can be matched in this way, then the region $R$ is 
a witness to a violation of \Cref{prob:scalar_lower_bound} as per 
\Cref{prop:solving_lower_bound}. If, however, a negative hyperplane of $R$ is 
found to belong to \emph{no} $s_j$, then the iteration terminates, since the 
region $R$ cannot be a witness to a violation of 
\Cref{prob:scalar_lower_bound}. This algorithm amounts to a \emph{greedy} 
matching of the negative hyperplanes of $R$, and it works by effectively 
examining the smallest intersection term to which the region $R$ can be a  
witness. The pseudocode for this algorithm, with an outer loop to iterate over 
regions of $\mathfrak{h}_a$, appears as \texttt{verifyScalarLB} in 
\Cref{alg:solve_problem_lower_bound}.

\IncMargin{0.5em}
\begin{algorithm}[!t]


\SetKwData{false}{False}
\SetKwData{true}{True}
\SetKwData{feas}{Feasible}

\SetKwData{constraints}{constraints}
\SetKwData{status}{status}
\SetKwData{sol}{sol}
\SetKwData{sollist}{solnList}
\SetKwData{actconstr}{actConstr}

\SetKwData{ha}{h\_a}
\SetKwData{reg}{reg}
\SetKwData{neghypers}{negHypers}
\SetKwData{i}{i}
\SetKwData{j}{j}

\SetKwFunction{append}{append}
\SetKwFunction{verifyScalarLB}{verifyScalarLB}
\SetKwFunction{solvefeas}{SolveLinFeas}
\SetKwFunction{findint}{FindInt}
\SetKwFunction{all}{all}
\SetKwFunction{regions}{Regions}
\SetKwFunction{L}{l}
\SetKwFunction{getneghypers}{NegativeHyperplanes}

\SetKw{Break}{break}
\SetKw{NOT}{not}
\SetKw{foriter}{for}
\SetKw{OR}{or} 
\SetKw{IN}{in}
\SetKw{CONT}{continue}

\SetKwInOut{Input}{input}
\SetKwInOut{Output}{output}

\Input{
\hspace{3pt}$a \in \mathbb{R}$, a lower bound to verify \\
\hspace{3pt}$\Xi_{\scriptscriptstyle N,M}$, parameters of a TLL NN to verify \\
\hspace{3pt}$L_X = \{\ell_{X,1}, \dots, \ell_{X,\mathsf{N}_X}\}$, affine functions \\
\hspace{15pt}specifying an input constraint polytope, $P_X$\\
}
\Output{
\hspace{3pt}Boolean ($\mathtt{True = SAT}; \mathtt{False = UNSAT}$)
}
\BlankLine
\SetKwProg{Fn}{function}{}{end}%
\Fn{\verifyScalarLB{$a$, $\Xi_{\scriptscriptstyle N,M}$, $L_X$}}{


	\ha $\leftarrow$ $[ \mathcal{L}_i\subarg{W_\ell,b_\ell}(x) - a\;\; $ \foriter $~ i = 1 \dots N~]$

	\For{\reg \IN \regions{\ha}}{

		\If{\reg $\cap P_X == \emptyset$}{
			\CONT \reg
		}

		\neghypers $\leftarrow$ \getneghypers{\reg}

		\For{\i \IN \neghypers}{

			\For{\j $= 1 \dots M$}{

				\If{\L{\i,a} $\not\in s_{\text{\j}}$}{

					\tcc{This region cannot be a witness to violation; go to next region}

					\CONT \reg

				}

			}

		\tcc{This region witnesses a violation; return}	

		\Return \false

		}
	}
	\Return \true
}
\caption{\texttt{verifyScalarLB}; i.e., solve \Cref{prob:scalar_lower_bound}}
\label{alg:solve_problem_lower_bound}
\end{algorithm}
\DecMargin{0.5em}




\subsection{On the Complexity of \myalg} 
\label{sub:on_the_complexity_of_myalg}
Given the remarks prefacing equation \eqref{eq:prob1_as_1ab}, it suffices to 
consider the complexity of \Cref{prop:solving_upper_bound} and 
\Cref{prob:scalar_lower_bound} individually. To simplify the notation in this 
section, we denote the complexity of running a linear program with $N$ 
constraints in $n$ variables by $\mathtt{LP}(N,n)$. \textbf{Note also: we 
consider complexities for a fixed $n$.}

\subsubsection{Complexity of \Cref{prob:scalar_upper_bound}} 
\label{ssub:complexity_of_scalar_upper_bound}
Analyzing the complexity of \texttt{verifyScalarUB} in 
\Cref{alg:solve_problem_upper_bound} is straightforward. There are $M$ total 
$\min$ (or intersection) terms, and each of these requires: one LP to check for 
$\geq b$ feasibility (line 5 of \Cref{alg:solve_problem_upper_bound}); followed 
by at most $N$ LPs to find an interior point (line 8 of 
\Cref{alg:solve_problem_upper_bound}). Thus, the complexity of  
\texttt{verifyScalarUB} is bounded by the following:
\begin{equation}
	O(
		M \cdot N \cdot \mathtt{LP}(N + \mathsf{N}_X, n).
	)
\end{equation}

\subsubsection{Complexity of \Cref{prob:scalar_lower_bound}} 
\label{ssub:complexity_of_scalar_lower_bound}
Analyzing the runtime complexity of \texttt{verifyScalarLB} in 
\Cref{alg:solve_problem_lower_bound} is also more or less straightforward, 
given an algorithm that enumerates the regions of a hyperplane arrangement.  
\myalg~uses an algorithm very similar to the reverse search algorithm described 
in \cite{AvisReverseSearchEnumeration1996} and improved slightly in 
\cite{FerrezCutsZonotopesArrangements2001}. For a hyperplane arrangement 
consisting of $N$ hyperplanes in dimension $n$, that reverse search algorithm 
has a per-region complexity bounded by:
\begin{equation}\label{eq:per-region_enumeration_cost}
	O(
		N \cdot \mathtt{LP}(N,n)
	).
\end{equation}
By \Cref{thm:arrangement_regions_bound} in \Cref{sec:preliminaries}, there are 
at most $O(N^n / n!)$ such regions.

Indeed, the per-region complexity of the loops in \texttt{verifyScalarLB} is 
easily seen to be bounded by $M \cdot N^2$ operations per region. Thus, it 
remains to evaluate the complexity of checking the intersection $\mathtt{reg} 
\cap P_X ==  \emptyset$ (see line 4 of  \Cref{alg:solve_problem_lower_bound}); 
however, this only appears as a separate operation for pedagogical simplicity. 
\myalg~actually follows the technique in 
\cite{FerlezBoundingComplexityFormally2020} to achieve the same assertion: the 
hyperplanes describing $P_X$ are added to the arrangement $\mathfrak{h}_a$, and 
any region for which one of those hyperplanes satisfies 
$\mathfrak{s}(\ell_{X,i}) = +1$ is ignored. This can be done with the 
additional per-region complexity associated with the size of the  larger 
arrangement, but without increasing the number of regions evaluated beyond  
$O(N^n/n!)$. Thus, we claim the complexity of 
\Cref{alg:solve_problem_lower_bound} is bounded by:
\begin{multline}
\label{eq:prob1b_complexity}
	O\big(\;
		M \cdot N^2 \cdot (N + \mathsf{N}_X) \cdot \mathtt{LP}(N + \mathsf{N}_X,n)) \cdot (N^n / n!)
	\;\big) \\
	= O(M \cdot N^{n+3} \cdot \mathtt{LP}(N + \mathsf{N}_X,n) / n!).
\end{multline}

\subsubsection{Complexity of \myalg~Compared to \cite{FerlezBoundingComplexityFormally2020}} 
\label{ssub:complexity_compared_to_ferlezboundingcomplexityformally2020}
We begin by adapting the TLL verification complexity reported in \cite[Theorem 
3]{FerlezBoundingComplexityFormally2020} to the scalar TLLs and single output 
properties of \Cref{prob:scalar_upper_bound} and 
\Cref{prob:scalar_lower_bound}. In the notation of this paper, it is as  
follows:
\begin{equation}
	O(
		n \cdot M \cdot N^{2n + 3} \cdot \mathtt{LP}(N^2 + \mathsf{N}_X,n) / n!
	).
\end{equation}
It is immediately clear that \myalg~has a significant complexity advantage for 
either type of property. Even the more expensive \texttt{verifyScalarLB} has a 
runtime complexity of $\approx O(N^n)$ compared to $\approx O(N^{2n})$ for 
\cite[Theorem 3]{FerlezBoundingComplexityFormally2020}, and that doesn't even 
count the larger LPs used in \cite[Theorem 
3]{FerlezBoundingComplexityFormally2020}.



%% file: implementation.tex

\section{Implementation} 
\label{sec:implementation}

\subsection{General Implementation} 
\label{sub:general_implementation}
The core algorithms of \myalg, \Cref{alg:solve_problem_upper_bound} and 
\Cref{alg:solve_problem_lower_bound}, are amenable to considerable parallelism. 
Thus, in order to make \myalg~as fast as possible, its implementation is 
focused on \emph{parallelism and concurrency} as much as possible.

With this in mind, \myalg~is implemented using a high-performance concurrency 
abstraction library for Python called \texttt{charm4py}  
\cite{Charm4pyCharm4pyDocumentation}. \texttt{charm4py} uses an actor model to 
facilitate concurrent programming, and it provides a number of helpful features 
to achieve good performance with relatively little programming effort. For 
example, it employs a cooperative scheduler to eliminate race-conditions, and 
it transparently offers the standard Python pass-by-reference semantics for 
function calls on the same Processing Element (PE). Moreover, it can be 
compiled to run on top of Message Passing Interface (MPI), which allows a 
single code base to scale from an individual multi-core computer to a 
multi-computer cluster. \myalg~was written with the intention of being deployed 
this way: it offers flexibility in how its core algorithms are assigned PEs, so 
as to take better take advantage of both compute and memory resources that are 
spread across multiple computers.

\subsection{Implementation Details for \Cref{alg:solve_problem_lower_bound}} 
\label{sub:implementation_details_for_alg:solve_problem_lower_bound}
Between the two core algorithms of \myalg, \Cref{alg:solve_problem_lower_bound} 
is the more challenging to parallelize. Indeed,  
\Cref{alg:solve_problem_upper_bound} has a trivial parallel implementation, 
since it consists primarily of a for loop over a known index set. In 
\Cref{alg:solve_problem_lower_bound}, it is the for loop over \emph{regions of 
a hyperplane} that makes parallelization non-trivial. Hence, this section 
describes how \myalg~parallelizes the region enumeration of a hyperplane 
arrangement for \Cref{alg:solve_problem_lower_bound}.

\begin{figure}[!t]
	\centering 
	\includegraphics[width=0.32\textwidth]{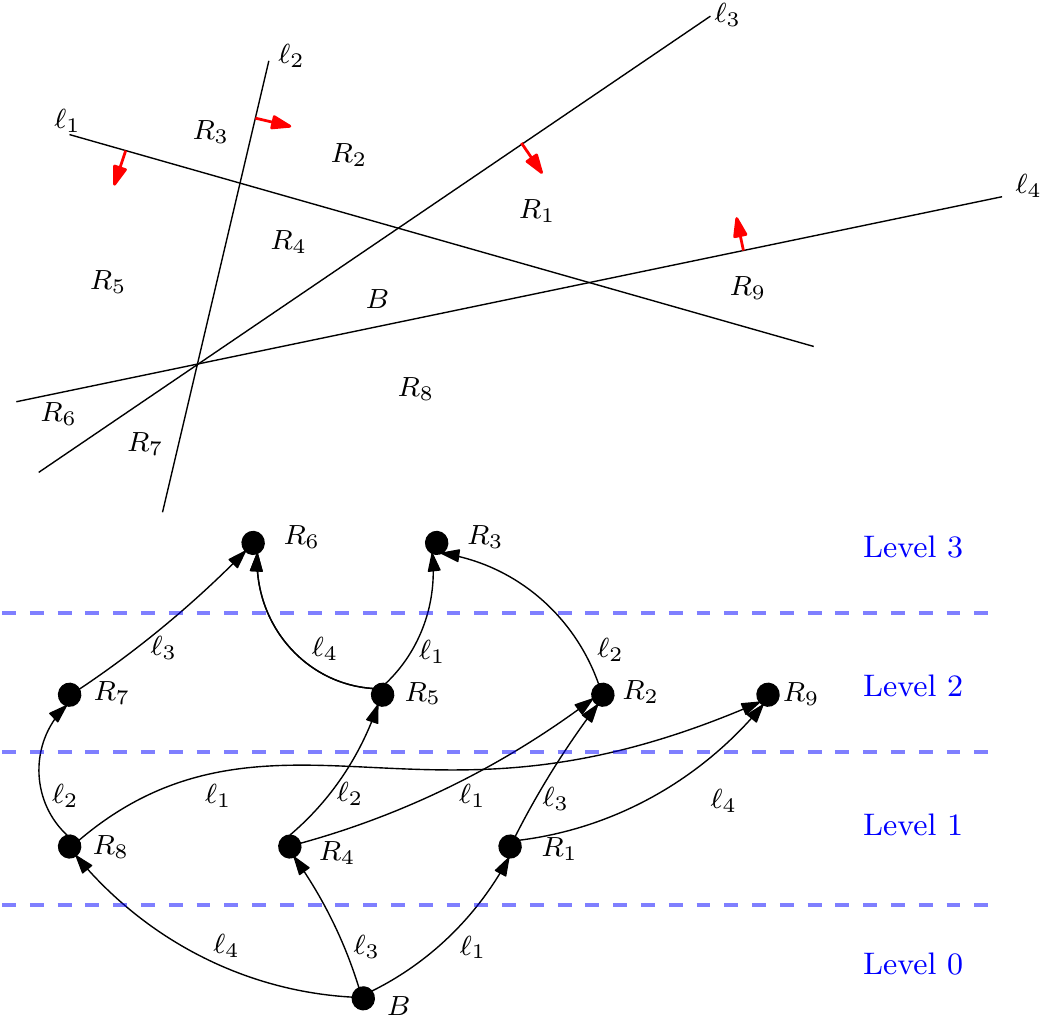} %
	\caption{(Top) A hyperplane arrangement $\{H^0_{\ell_1}, \dots, H^0_{\ell_4} \}$ with positive half spaces denoted by red arrows and regions $B, R_1, \dots, R_9$. (Bottom) The corresponding adjacency poset. }\vspace{-3mm}
	\label{fig:hyperplane_arrangement} %
\end{figure}

To describe the architecture of \myalg's implementation of hyperplane region  
enumeration, we first briefly introduce the well-known reverse search algorithm 
for the same task \cite{FerrezCutsZonotopesArrangements2001}, and the  
algorithm on which \myalg's implementation is loosely based. As its name 
suggests, it is a \emph{search} algorithm: that is, it starts from a known 
region of the arrangement and searches for regions adjacent to it, and then 
regions adjacent to those, and so on. In particular, though, the reverse search 
algorithm is fundamentally a depth-first search, and it uses a minimum-index 
rule to ensure that regions are not visited multiple times (i.e., functioning 
much as Bland's rule in simplex  solvers) 
\cite{FerrezCutsZonotopesArrangements2001, AvisReverseSearchEnumeration1996}. 
This type of search algorithm has the great benefit that it is memory 
efficient, since it tracks the current state of the search using only the 
memory required to track adjacency indices; even the information required to 
back-track over the current branch is computed rather than being stored  
\cite[pp. 10]{FerrezCutsZonotopesArrangements2001}. However, the memory 
efficiency of \cite{FerrezCutsZonotopesArrangements2001} comes at the expense 
of parallelizability, precisely because the search state is stored using 
variables that are incremented with each descent down a branch.

The natural way to parallelize such a search is to allow multiple concurrent 
search workers, but have them enter their independent search results into a 
common, synchronized hash table. Assuming an amortized $O(1)$ hash complexity, 
this solution doesn't affect the overall computational complexity; on the other 
hand, it comes with a steep memory penalty, since it could require storing all 
$O(N^n)$ regions in the worst case. However, there is nevertheless a way to 
efficiently enable and coordinate multiple search processes, while avoiding 
this excessive memory requirement.

To this end, \myalg~leverages a special property of the region adjacency 
structure in a hyperplane arrangement. In particular, the regions of a 
hyperplane arrangement can be organized into a \emph{leveled} adjacency poset 
\cite{EdelmanPartialOrderRegions1984}. That is, relative to any initial base 
region, all of the regions in the arrangement can be grouped according to the 
number of hyperplanes that were ``crossed'' in the process discovering them;  
the same idea is also implicit in \cite{FerrezCutsZonotopesArrangements2001, 
AvisReverseSearchEnumeration1996}. This leveled property of the adjacency poset 
is illustrated in \Cref{fig:hyperplane_arrangement}: the top pane shows a 
hyperplane arrangement with its regions labeled; the bottom pane depicts the 
region adjacency poset for this arrangement, with levels indicated relative to 
a base region, $B$. For example, a search starting from $B$ will find region 
$R_1$ by crossing  \emph{only} $\ell_1$ and region $R_2$ by crossing $\ell_1$ 
and $\ell_3$.

Thus, \myalg~still approaches the region enumeration problem as search, but 
instead it proceeds level-wise. All of the regions in the \emph{current} level 
can be easily divided among the available processing elements, which then 
search in parallel for their immediately adjacent regions; the result of this 
search is a list of regions comprising the entire \emph{next} level in the 
adjacency poset, which then becomes the current level and the process repeats. 
From an implementation standpoint, searching the region adjacency structure 
level-wise offers a useful way of reducing \myalg's memory footprint. In 
particular, once a level is fully explored, the regions it contains \emph{will 
never be seen again}. Thus, \myalg~need only maintain a hash of regions from 
one level at a time: the hash tables from previous levels can be safely 
discarded. \emph{In this way, \myalg~achieves a parallel region search but 
without resorting to hashing the entire list of discovered regions}.



Finally, we note that a search-type algorithm for region enumeration has a 
further advantage for solving a problem like \Cref{prob:scalar_lower_bound}, 
though. In particular, a search algorithm reveals each new region with a 
relatively low computational cost ---  see 
\eqref{eq:per-region_enumeration_cost}; this is in contrast to some other  
enumeration algorithms, which must \emph{run to completion before even  
{\bfseries one region} is available}. Since 
\Cref{alg:solve_problem_lower_bound} is structured such that it can terminate 
on the first violating region found, this has a considerable advantage for 
\textbf{UNSAT} problems, as a violating region may be found very early in the 
search.



%% file: experiments.tex

\section{Experiments} 
\label{sec:experiments}
We conducted a series of experiments to evaluate the performance and 
scalability of \myalg~as a TLL verifier, both in its own right and relative to 
general NN verifiers applied to TLL NNs. In particular, we conducted the 
following three experiments:
\begin{enumerate}
	\item[Exp. 1)] Scalability of \myalg~as a function of TLL input 
		dimension, $n$; the number of local linear functions, $N$, and the 
		number of selector sets, $M$,  remained fixed.

	\item[Exp. 2)] Scalability of \myalg~as a function of the number of 
		local linear functions, $N$, with $N = M$; the input dimension, $n$, 
		remained fixed.

	\item[Exp. 3)] Comparison with general NN verifiers: nnenum 
		\cite{BakImprovedGeometricPath2020}, PeregriNN 
		\cite{KhedrPEREGRiNNPenalizedRelaxationGreedy2020} and Alpha-Beta-Crown 
		\cite{wang2021betacrown}.
\end{enumerate}
\noindent All experiments were run within a VMWare Workstation Pro virtual 
machine (VM) running on a Linux host with 48 hyper-threaded cores and 256 GB of 
memory. All instances of the VM were given 64 GB of memory, and a core count 
that is specified within each experiment. A timeout of 300 seconds was used  in 
all cases.

\subsection{Experimental Setup: Networks and Properties} 
\label{sub:setup}

\subsubsection{TLL NNs Verified} 
\label{ssub:tll_nns_verified}
Given that \Cref{prob:tll_verification_multi} can be decomposed into instances 
of \Cref{prob:scalar_upper_bound} and \Cref{prob:scalar_lower_bound}, all of 
these experiments were conducted on scalar-output TLL NNs using real-valued 
properties of the form in either \Cref{prob:scalar_upper_bound} or 
\Cref{prob:scalar_lower_bound}.

In Experiments 1 and 2, TLL NNs of the desired $n$, $N$ and $M$ were generated 
randomly according to the following procedure, which was designed to ensure 
that they are unlikely to be degenerate on (roughly) the input set $[-1,  
1]^n$. The procedure is as follows:
\begin{enumerate}
	\item Randomly generate elements of $W_\ell$ and $b_\ell$ according to 
		normal distributions of mean zero and standard deviations of $1/10$ and 
		$1$, respectively.

	\item Randomly generate selector sets, $s_j$, by generating random 
		integers between $0$ and $2^{N+1}-1$, and continue generating them by 
		this mechanism until $M$ are obtained such that no two selector sets 
		satisfy $s_j \subseteq s_{j^\prime}$ (a form of degeneracy).

	\item For each corresponding selector matrix, $S_j$, solve $M$ instances 
		of the following least squares problem:
		\begin{equation}
			\min_{x_j \in \mathbb{R}^n} \lVert S_j W_\ell x_j - S_j b_\ell  \rVert_2,
		\end{equation}
		to obtain the $M$ vectors, $x_1, \dots, x_M$.

	\item  Then scale each row of $W_\ell$ (from (1) above) by the 
		corresponding row of the vector:
		\begin{equation}
			\left[
				\begin{smallmatrix}
					\max_{j=1\dots M} | \llbracket x_j \rrbracket_{[1,1]} |
					&
					\dots
					&
					\max_{j=1\dots M} | \llbracket x_j \rrbracket_{[N,1]} |
				\end{smallmatrix}
			\right]^\text{T}.
		\end{equation}
		This has the (qualitative) effect of forcing the mutual intersections 
		of randomly generated local linear functions to be concentrated around 
		the origin.
\end{enumerate}

In Experiment 3, we obtained and used the scalar TLL NNs that were tested in 
\cite{FerlezBoundingComplexityFormally2020}. These networks all have $n=2$ and 
$N=M$; there are thirty examples for each of the sizes $N = M = 8, 16, 24, 32, 
40, 48, 56$ and $64$ (each size has a common neuron count, ranging from $256$  
neurons for $N=8$ to $16384$ neurons for $N=64)$. We used these networks in 
particular so as to enable some basis of comparison with the experimental 
results in \cite{FerlezBoundingComplexityFormally2020}. This is relevant, since 
that tool is not publicly available, and hence omitted from our comparison. 
\emph{Note: we considered these networks with different, albeit similar, 
properties to  those used in \cite{FerlezBoundingComplexityFormally2020}; see 
\Cref{ssub:properties_verified} below.}

\subsubsection{Input Constraints, $P_X$} 
\label{ssub:properties_verified}
In all experiments, we considered verification problems with $P_X = [-2,2]^n$.  
For the TLLs we generated, there is no great loss of generality in considering 
this fixed size for $P_X$, since we generated them to be ``interesting'' in 
this vicinity; see \Cref{ssub:tll_nns_verified}. However, using a 
hyper-rectangle $P_X$ was \emph{necessary} for Experiment 3, since some of the 
general NN verifiers accept only hyper-rectangle input constraints. Thus, we 
made the universal choice $P_X = [-2,2]^n$ for consistency between  
experiments.

Note, however, that \cite{FerlezBoundingComplexityFormally2020} verified 
general polytopic input constraints on the networks we borrowed for Experiment 
3. Nevertheless, we expect the results for \myalg~in Experiment 3 to be 
somewhat comparable to the results in 
\cite{FerlezBoundingComplexityFormally2020}, since all of those polytopic 
constraints are \emph{contained} in the box, $P_X = [-2,2]^2$.

\subsubsection{Output Properties Verified} 
\label{ssub:properties_verified}
For a scalar TLL, only two parameters are required to specify an output  
property: a real-valued scalar and the direction of the inequality. In all 
cases, the direction of the inequality was determined by the outcome of 
Bernoulli random variable. And in all cases except one (noted in Experiment 2), 
the random real-valued property was generated by the following procedure.  
First, the TLL network was evaluated at 10,000 samples collected from the set 
$P_X$; any property between the min and max of these output samples is 
guaranteed to be \textbf{UNSAT}. Then, to get a mixture of 
\textbf{SAT}/\textbf{UNSAT} properties, we select a random property from this 
interval symmetrically extended to twice its original size.

\subsection{Experiment 1: Input Dimension Scalability} 
\label{sub:experiment_1_input_dimension_scalability}
In this experiment, we evaluated the scalability of \myalg~as a function of 
input dimension of the TLL to be verified. 
To that end, we generated a suite of TLL NNs with input sizes varying from 
$n=1$ to $n=30$, using the procedure described in 
\Cref{ssub:tll_nns_verified}. We generated 20 instances for each size, and for 
all TLLs, we kept $N=M=64$ constant. We then verified each of these TLLs with 
respect to its own, randomly generated property (as described in 
\Cref{ssub:properties_verified}). For this experiment, \myalg~was run in a VM 
with 32 cores.

\Cref{fig:experiment1} summarizes the results of this experiment with a 
box-and-whisker plot of verification times: each box-and-whisker\footnote{As 
usual, the boxes denote the first and third quartiles; the orange horizontal 
line denotes the median; and the whiskers show the maximum and minimum.} 
summarizes the verification times for the twenty networks of the corresponding 
input dimension; no properties/networks resulted in a timeout. The data in the 
figure shows a clear trend of increasing median, as expected for progressively 
harder problems (recall the runtime complexities indicated in 
\Cref{sub:on_the_complexity_of_myalg}). By contrast, note that the minimum and 
maximum runtimes grow very slowly with dimension: given the complexity analysis 
of \Cref{sub:on_the_complexity_of_myalg}, we speculate that these results are 
likely due to the characteristics of the generated TLLs. That is, the 
generation procedure appears to ``saturate'' in the sense that it eventually 
produces networks which require, on average, a constant number of loop 
iterations to verify.


\begin{figure}[t]
	\centering %
	\includegraphics[width=0.40\textwidth,trim={0.72in 0.92in 1in 1.125in},clip]{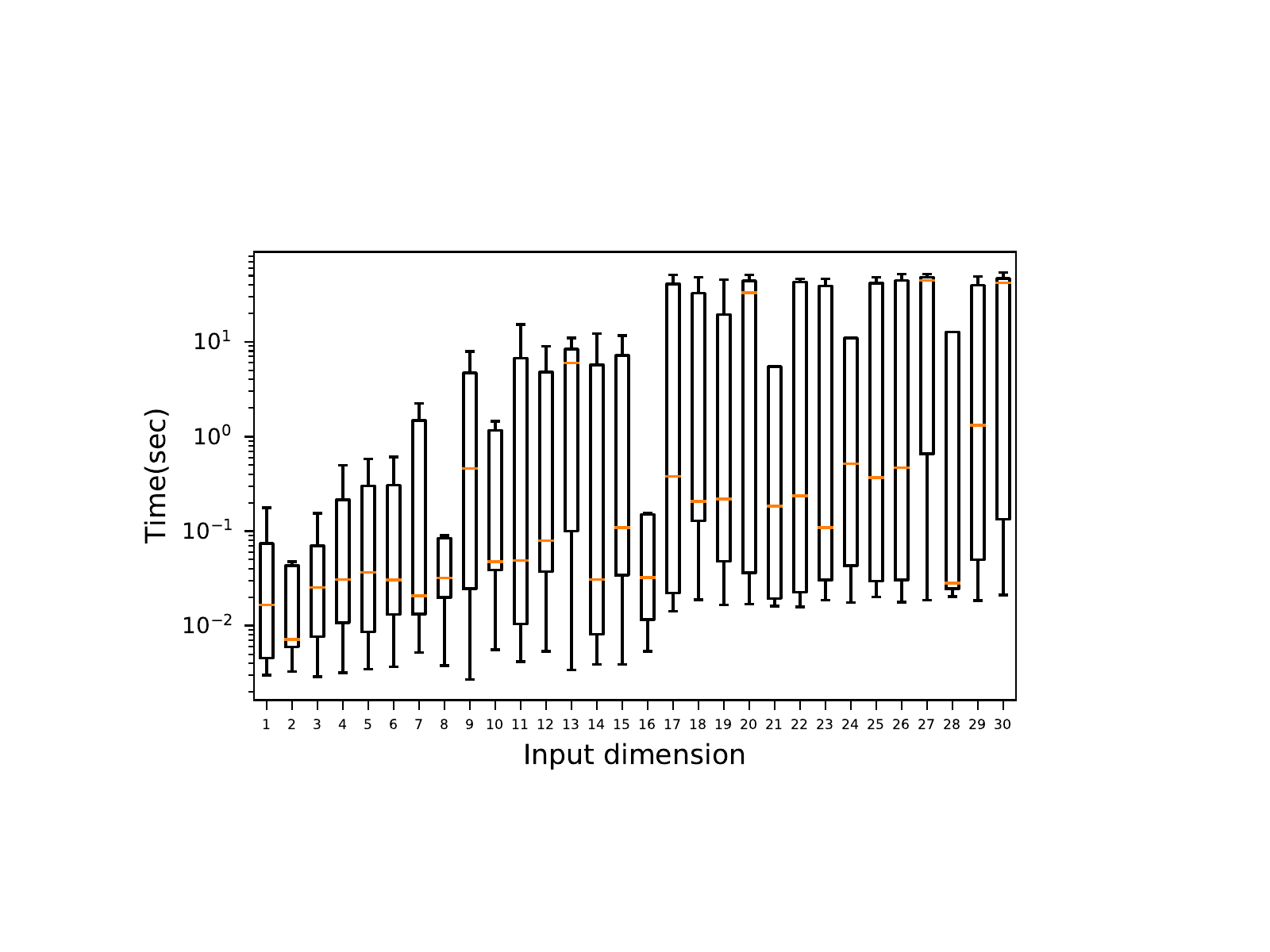} %
	\vspace{-2mm}
	\caption{Experiment 1. Scaling the Input Dimension} %
	\vspace{-5mm}
	\label{fig:experiment1} %
\end{figure}


\subsection{Experiment 2: Network Size Scalability} 
\label{sub:experiment_2_network_size_scalability}
In this experiment, we evaluated the scalability of \myalg~as a function of the 
number of local linear functions, $N$, in the TLL to be verified. To that end, 
we generated a suite of of TLL NNs local linear functions ranging in  number 
from $N=16$ to $N=512$, again using the procedure described in 
\Cref{ssub:tll_nns_verified}. We generated 20 instances for each value of $N$, 
and for all networks we set $M=N$ and $n=15$. We then verified each of these 
networks with respect to its own, randomly generated property. The properties 
for network sizes $N=16$ through $N=256$ were generated as described in 
\Cref{ssub:properties_verified}; however, a our TensorFlow implementation 
occupied too much memory to generate samples for TLLs of size $N=512$, so the 
properties for these networks were generated using the bounds for the $N=256$  
TLLs. For this experiment, \myalg~was run in a VM with 32 cores.


\Cref{fig:experiment2} summarized the results of this experiment with a 
box-and-whisker plot of verification times: each box-and-whisker summarize the 
verification times for the twenty test cases of the corresponding size, much as 
in \Cref{sub:experiment_1_input_dimension_scalability}. However, since some 
verification problems timed out in this experiment, those time outs were 
excluded from the box-and-whisker; they are instead indicated by a superimposed 
bar graph, which displays a count of the number of timeouts obtained from each 
group of equally-sized TLLs. The data in this figure shows the expected trend of 
increasingly difficult verification as $N$ increases; this is especially 
captured by the trend of experiencing more timeouts for larger networks. The 
outlier to this trend is the size $N=512$, but this is most likely due to 
different method of generating properties for these networks (see above). 
Finally, note that the minimum verification times exhibit a much slower growth 
trend, as in Experiment 1.


\begin{figure}[t]
	\centering %
	\includegraphics[width=0.40\textwidth,trim={0.72in 0.9in 0.85in 1.125in},clip]{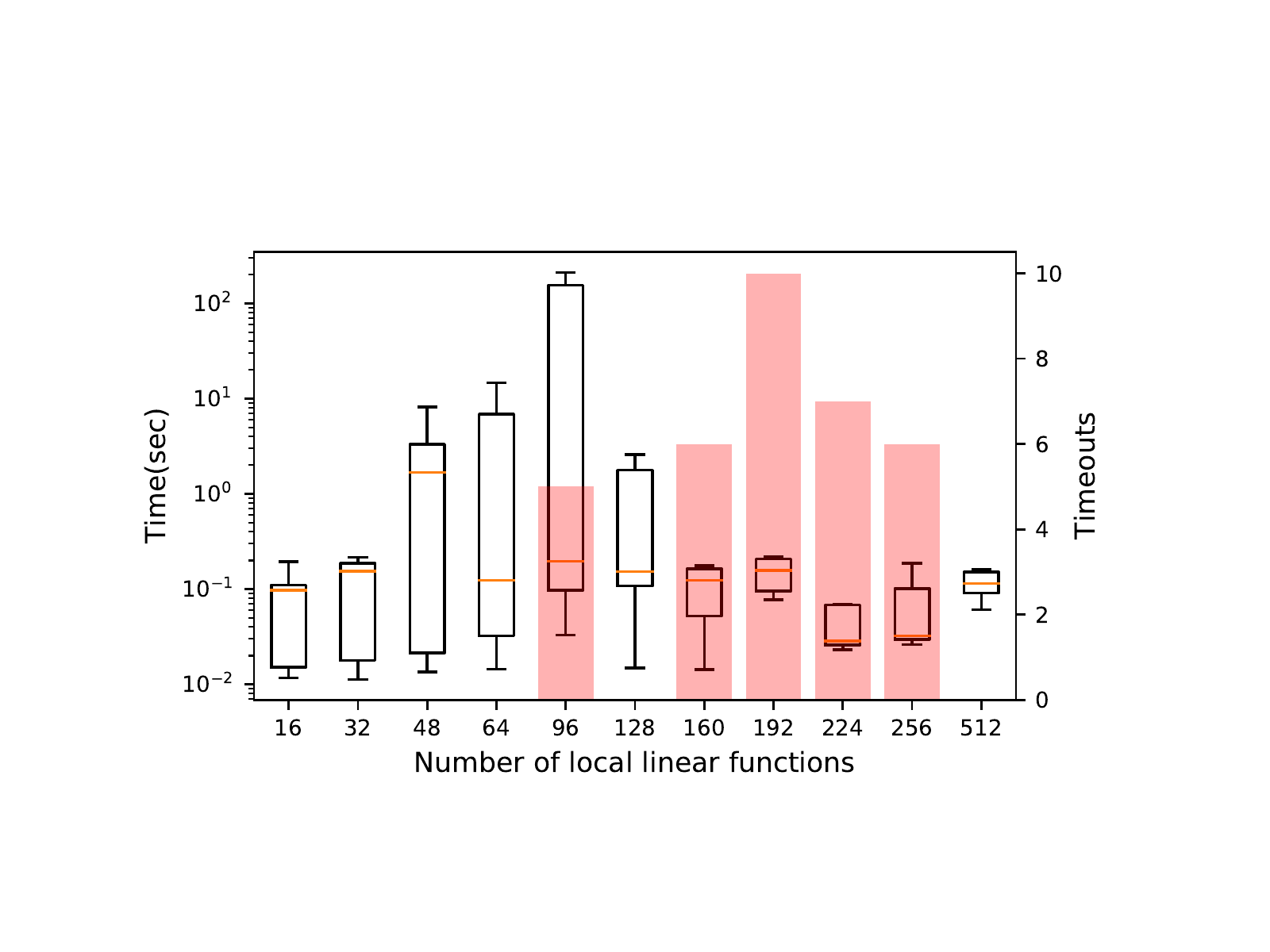} %
	\caption{Experiment 2. Scaling the Number of Local Linear Functions} %
	\vspace{-5mm}
	\label{fig:experiment2} %
\end{figure}


\subsection{Experiment 3: General NN Verifiers} 
\label{sub:experiment_3_comparison_with_general_nn_verifiers}
In this experiment, we compared the verification performance of \myalg~with 
state-of-the-art (SOTA) NN verifiers designed to work on general deep NNs. For 
this experiment, we compared against generic verifiers Alpha-Beta-Crown 
\cite{wang2021betacrown}, nnenum \cite{BakImprovedGeometricPath2020} and 
PeregriNN \cite{KhedrPEREGRiNNPenalizedRelaxationGreedy2020}, as a 
representative sample of SOTA NN verifiers. Moreover, we conducted this 
experiment on the same 240 networks used in 
\cite{FerlezBoundingComplexityFormally2020}, and described in 
\Cref{ssub:tll_nns_verified}; this further facilitates a limited comparison 
with that algorithm, subject to the caveats described in 
\Cref{ssub:tll_nns_verified} and \Cref{ssub:properties_verified}.

In order to make this test suite of TLLs available to the generic verifiers, 
each network was first implemented as a TensorFlow model using a custom 
implementation tool. The intent was to export these TensorFlow models to the 
ONNX format, which each of the generic verifiers can read. However, most of the 
generic verifiers do not support implementing multiple feed-forward paths by 
tensor reshaping operations, as in the most straightforward implementation of a 
TLL; see \Cref{fig:tll_arch} and \Cref{sub:two_layer_lattice_neural_networks}. 
Thus, we had to first  ``flatten'' our TensorFlow implementation into an  
equivalent network where each $\min$ network accepts the outputs of \emph{all} 
of the selector matrices, only to null the irrelevant ones with additional 
kernel zeros in the first layer. This is highly sub-optimal, since it results 
in neurons receiving many more inputs than are really required. However, we 
could not devise another method to circumvent this limitation present in most 
of the tools.

\begin{figure}[t]
	\centering %
	\includegraphics[width=0.40\textwidth,trim={0.65in 0.85in 1in 1.125in},clip]{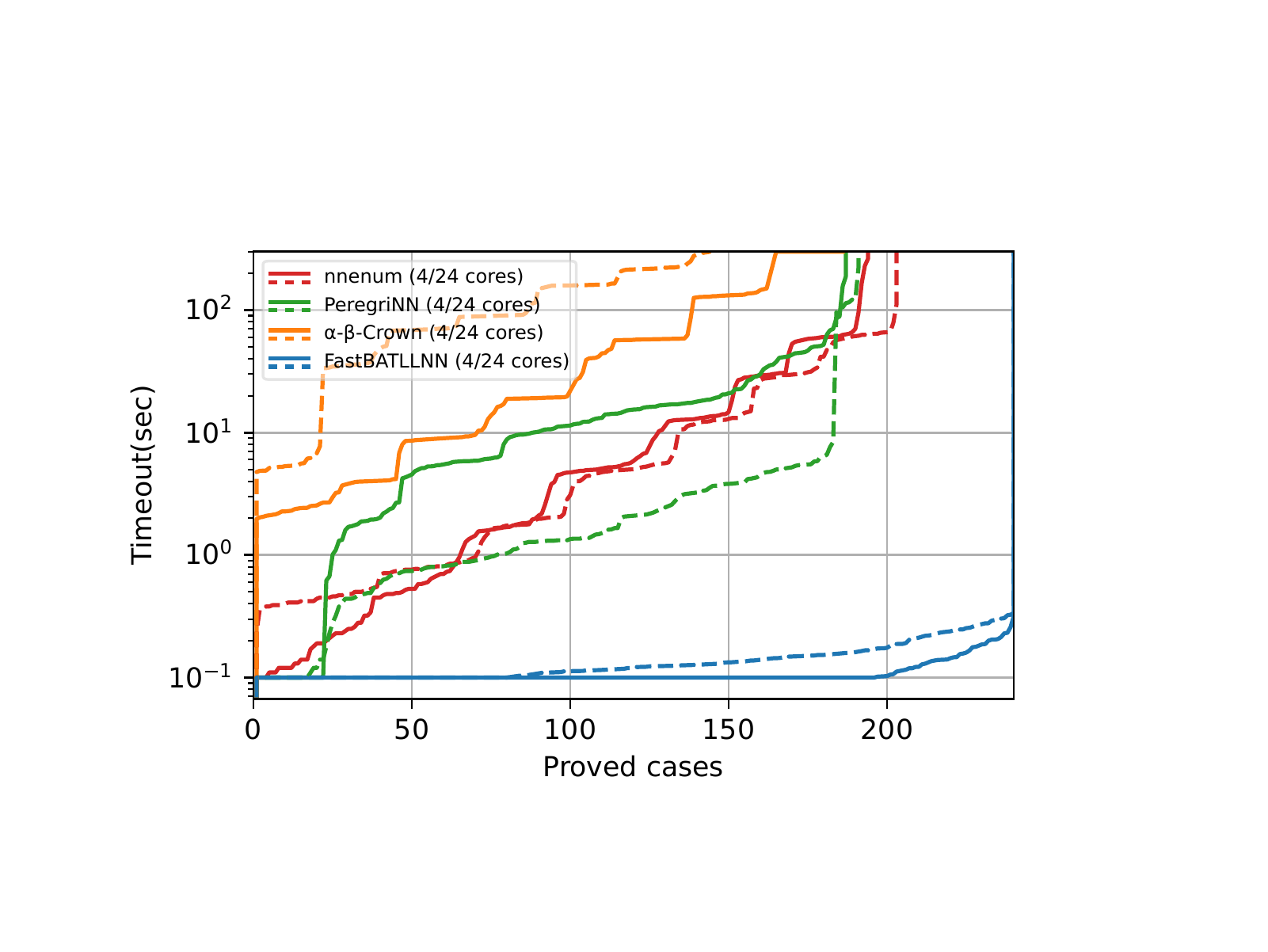} %
	\vspace{-2mm}
	\caption{Experiment 3. General NN Verifiers} %
	\vspace{-5mm}
	\label{fig:experiment3} %
\end{figure}

With all of the tools able to read the same NNs in our (borrowed) test suite, 
we randomly generated verification properties for each of the networks, as in 
the previous experiments. However, recall that the generic NN verifiers have a 
slightly different interpretation of properties compared with \myalg. For 
scalar-output networks, this amounts to verifying the properties in 
\Cref{prob:scalar_upper_bound} and \Cref{prob:scalar_lower_bound} with the same 
interpretations, but with \emph{non-strict} inequalities instead of  
\emph{strict} inequalities. Since this will only generate divergent results 
when a property happens to be exactly equal to the maximum or minimum value on 
$P_X$, we elide this issue.

Thus, we ran each of the tools, \myalg, Alpha-Beta-Crown, nnenum and PeregriNN 
on this test suite of TLLs and properties. PeregriNN was configured with 
$\texttt{SPLIT\_RES = 0.1}$; nnenum was configured with  
\texttt{TRY\_QUICK\_APPROX = True} and all other parameters set to default  
values; and  Alpha-Beta-Crown was configured with \texttt{input\ space\ 
splitting}, \texttt{share\_slopes = True},  \texttt{no\_solve\_slopes = True},  
$\texttt{lr\_alpha = 0.01}$, and \texttt{branching\_method = sb}. All the 
solvers used \texttt{float64} computations.  Furthermore, we ran two versions 
of this experiment, one where the VM had 4 cores and one where the VM had 24 
cores.

\Cref{fig:experiment3} summarizes the results of this experiment in the form of 
a cactus plot: a point on any one of the curves indicates the timeout that 
would be required to obtain the corresponding number of proved cases for that 
tool (from among of the test suite described above). As noted, each tool was 
run separately in two VMs, one with 4 cores and one with 24 cores; thus, each 
tool has two curves in \Cref{fig:experiment3}. The data shows that \myalg~is on 
average 960$\times$/435$\times$ faster than nnenum, 1800$\times$/1370$\times$ 
faster than Alpha-Beta-Crown, and 1000$\times$/500$\times$ faster than 
PeregriNN using 4 and 24 cores respectively. \myalg~also proved all 240 
properties in just 17 seconds (4 cores), whereas nnenum proved 193, 
Alpha-Beta-Crown proved 153, and PeregriNN proved 186. Note that unlike the 
other tools, \myalg~doesn't exhibit exponential growth in execution time on 
this test suite, which is consistent with the complexity analysis in 
\Cref{sub:on_the_complexity_of_myalg}. Despite the caveats noted  above, 
\myalg~also compares favorably with the execution times shown in \cite[Figure  
1(b)]{FerlezBoundingComplexityFormally2020}, which end up in the 100's or 
1000's of seconds for $N=64$. Finally, note that although \myalg~suffered from 
slightly worse measured performance with 24 cores, the rate at which its 
timeouts increased was significantly slower than with 4 cores; this suggests 
the data is reflecting constants, rather than inefficient use of parallelism. 
Of similar note, Alpha-Beta-Crown seems to suffer from the overhead of using 
more CPU cores. Based on our understanding, the algorithm doesn't benefit from 
multiple cores except for solving MIP problem. 